\documentclass[Afour,sageh,times]{sagej}

\usepackage{moreverb,url}

\usepackage[colorlinks,bookmarksopen,bookmarksnumbered,citecolor=red,urlcolor=red]{hyperref}

\usepackage{url}

\usepackage[utf8]{inputenc}
\usepackage{graphicx}
\usepackage{amsmath,amsthm}
\usepackage{amssymb,xcolor}
\usepackage{mathtools}
\usepackage{cancel}
\usepackage{enumitem}
\usepackage[noend]{algorithmic}
\usepackage{algorithm}
\usepackage{wrapfig}
\usepackage{hyperref}
\usepackage{tikz-cd}
\usepackage{tikz}
\usetikzlibrary{automata, positioning}

\usepackage[scriptsize]{subfigure}
\usepackage{epsfig} 

\usepackage{xcolor}
\definecolor{weinrot}{rgb}{0.7,0.15,0.15}

\newcommand{\BS}[1]{{{\color{black}{#1}}}}

\newcommand{\KT}[1]{{{\color{black}{#1}}}}
\newcommand{\new}[1]{{{\color{black}{#1}}}}
\newcommand{\newnew}[1]{{{\color{black}{#1}}}}

\newcommand*{\T}{\mathcal{T}}

\newtheorem{proposition}{Proposition}
\newtheorem{remark}{Remark}
\newtheorem{problem}{Problem}

\newtheorem{claim}{Claim}

\theoremstyle{definition}
\newtheorem{example}{Example}
\newtheorem{definition}{Definition}

\newcommand\BibTeX{{\rmfamily B\kern-.05em \textsc{i\kern-.025em b}\kern-.08em
T\kern-.1667em\lower.7ex\hbox{E}\kern-.125emX}}

\def\volumeyear{}

\def\Re{{\mathbb{R}}}

\def\histx{{\tilde{x}}}

\def\histy{{\tilde{y}}}
\def\histY{{\tilde{Y}}}
\def\histu{{\tilde{u}}}
\def\histU{{\tilde{U}}}

\def\E{{\cal{E}}}
\def\N{{\mathbb{N}}}
\def\R{{\cal{R}}}

\def\is{{\iota}}
\def\his{{\eta}}
 
\def\imap{{\kappa}}

\def\fmaphist{{\phi_{hist}}}
\def\Fmaphist{{\Phi_{hist}}}

\def\xdot{{\dot x}}

\def\K{{\cal K}}

\def\atan2{\operatorname{atan2}}
\def\pow{{\rm pow}}

\def\I{{\cal I}}
\def\ifs{{\cal I}} 
\def\Ifs{{\cal I}}

\def\ifshist{{\cal I}_{hist}}

\def\Ifsder{{\cal I}_{der}}

\def\Ifshist{{\cal I}_{hist}}

\def\Ifsmin{{\cal I}_{min}}
\def\Ifst{{\cal I}_{task}}

\def\imapb{{\kappa_{task}}}

\def\Ifstask{{\cal I}_{task}}

\newcommand{\cat}{{}^\frown}

\newcommand{\cG}{\mathcal{G}}
\newcommand{\cE}{\mathcal{E}}
\newcommand{\cT}{\mathcal{T}}
\newcommand{\Ifspsr}{\mathcal{I}_\textrm{PSR}}

\usepackage{acronym}
\acrodef{Ispace}[I-space]{\emph{information space}}
\acrodef{Istate}[I-state]{\emph{information state}}
\acrodef{Imap}[I-map]{\emph{information mapping}}
\acrodef{ITS}[ITS]{\emph{information transition system}}
\acrodef{ITSs}[ITSs]{\emph{information transition systems}}
\acrodef{DITS}[DITS]{\emph{deterministic information transition system}}
\acrodef{NITS}[NITS]{\emph{nondeterministic information transition system}}
\acrodef{POMDPs}[POMDPs]{\emph{partially observable Markov decision processes}}
\acrodef{PSRs}[PSRs]{\emph{predictive state representations}}

\begin{document}

\runninghead{Sakcak et al.}

\title{A Mathematical Characterization of Minimally Sufficient Robot
  Brains
}

\author{Basak Sakcak, Kalle G. Timperi, Vadim Weinstein, and Steven
  M. LaValle}

\affiliation{The authors are with Center for Ubiquitous Computing, Faculty of
  Information Technology and Electrical Engineering, University of
  Oulu, Finland.  {\tt\small (e-mail: firstname.lastname@oulu.fi)}}

\corrauth{Basak Sakcak, Center for Ubiquitous Computing Faculty of
  Information Technology and Electrical Engineering, University of
  Oulu, Finland.}

\email{basak.sakcak@oulu.fi}

\begin{abstract}
  This paper addresses the lower limits of encoding and processing the information
  acquired through interactions between an internal system (robot algorithms or software) and an external system (robot body and its
  environment) in terms of action and observation histories. Both
  are modeled as transition systems. We want to know the
  weakest internal system that is sufficient for achieving passive
  (filtering) and active (planning) tasks. We introduce the notion of an
  \ac{ITS} for the internal system which is a transition system over a
  space of information states that reflect a robot's or other observer's
  perspective based on limited sensing, memory, computation, and
  actuation. An \ac{ITS} is viewed as a filter and a policy or plan is
  viewed as a function that labels the states of this
  \ac{ITS}. Regardless of whether internal systems are obtained by
  learning algorithms, planning algorithms, or human insight, we want
  to know the limits of feasibility for given robot hardware and
  tasks. We establish, in a general setting, that minimal information
  transition systems (ITSs) exist up to reasonable equivalence
  assumptions, and are unique under some general conditions.  We then
  apply the theory to generate new insights into several problems,
  including optimal sensor fusion/filtering, solving basic planning
  tasks, and finding minimal representations for modeling a system
  given input-output relations. 
\end{abstract}

\maketitle

\section{Introduction}

Accomplishing a robot's tasks may involve designing or employing a combination of \BS{different parts:} planning algorithms, sensor fusion or filtering methods, machine learning algorithms, and control laws. Given a problem expressed in terms of a well-defined task structure, the relationship between these different parts
with each other, and \BS{their relation} with the task is often ignored. 
\BS{Each part is developed rather in isolation, heavily motivated by the long lasting traditions in
robotics.}
For example, navigating a mobile robot to a goal
configuration is typically achieved by estimating the robot configuration
in a known (or an unknown) map and applying a feedback policy that
relies on the estimated configuration. 
\BS{Indeed, for some problems it may be possible that a simpler filtering approach that does not require estimating the full state could as well be sufficient to achieve the given task.}
\BS{This may lead many to believe that robotics itself does not have its own, unique theoretical core (on this, we agree with \citep{Kod21}) and it appears as an application area for other fields; designing and testing machine learning algorithms, planning algorithms, sensor fusion methods, control laws, and so on.}
\BS{In our quest towards a theory that is unique to robotics and that plays a similar role to Turing machines for computer science, or $\xdot = f(x,u)$ over differentiable manifolds for control theory, we want to establish the limits of the intertwined notions of sensing, learning, filtering, and planning with respect to a given problem.}
We would like to have a general framework that allows
researchers to formulate and potentially answer questions such as: Does a
solution even exist to a given problem?  What are the minimal
necessary components to solve it?  What should the best learning
approach imaginable produce as a representation?  Such questions would
be analogous to existence and uniqueness in control and dynamical
systems, or decidability and complexity (especially Kolmogorov) in
theoretical computer science.

This paper proposes a mathematical robotics theory that is built from the input-output relationships between two (or more) coupled dynamical systems. For a programmable \BS{mechanical} system (robot) embedded in an environment, the input-output relationships correspond to sensing and actuation between two coupled systems; an \emph{internal system} (robot brain) and an \emph{external system} (robot body and the environment). This relation is shown in Figure~\ref{fig:extint}(a-b). 
We assume that the robot hardware is fixed, which means fixing the
sensors and the actuators, and we focus on
\newnew{determining which necessary and sufficient conditions the internal system has to maintain for a task to be accomplished.}
In light of these conditions, we try to find a
\emph{minimal sufficient} internal system which corresponds to the
weakest possible representation of the acquired information through
interactions; \BS{reducing} the internal system any further makes the
problem unsolvable.

At the core of our framework is the notion of an \ac{ITS} which builds on
the well-studied notion of {transition systems}. The
\emph{information} part of an ITS comes from {\em information spaces}
\cite[Chapter 11]{Lav06} which is developed as a foundation of planning
with imperfect state information due to sensing uncertainty.  The
concept of {\em sufficient information mappings} appears therein.  It
is generalized in this paper, and the state space of each ITS will in
fact be an information space. An internal system will be modeled as an
\ac{ITS} and its sufficiency and minimality with respect to a problem
will be analyzed using this framework.

We categorize tasks into two classes: active and passive. Informally,
a passive task corresponds to filtering and an active one corresponds
to planning or control. In our work an \ac{ITS} can be seen as a
filter and together with its underlying information space serves as a
domain over which a plan or a policy can be expressed and analyzed. We
will consider a variety of information spaces, which also encompass
robot configuration spaces or phase spaces, that are typically used for
planning tasks.  A distinction between \emph{model-based} and
\emph{model-free} formulations will be considered too, in line with the
choices commonly found in machine learning. We characterize the problems corresponding to the active and passive tasks and define notions of feasibility and minimality for \ac{ITSs} that solve these problems.
In our approach to
finding minimal sufficient \ac{ITSs}, we will analyze the limits of
reducing or collapsing the information spaces, until the lower limits of
task feasibility are reached.

\begin{figure*}[t]
\begin{center}
\begin{tabular}{ccc}
  \includegraphics[width=0.45\linewidth]{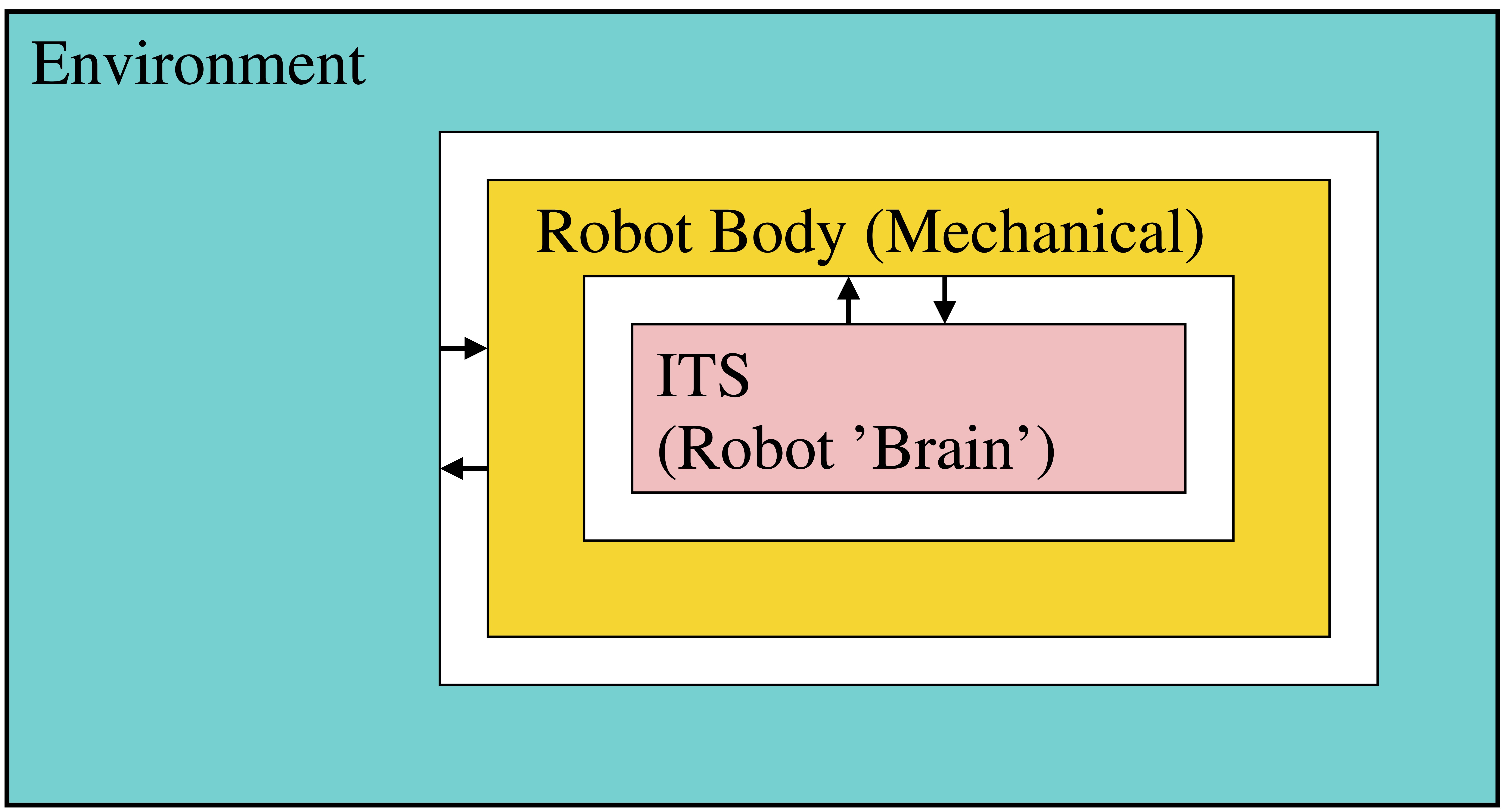} & \hspace*{0.5cm} & \includegraphics[width=0.45\linewidth]{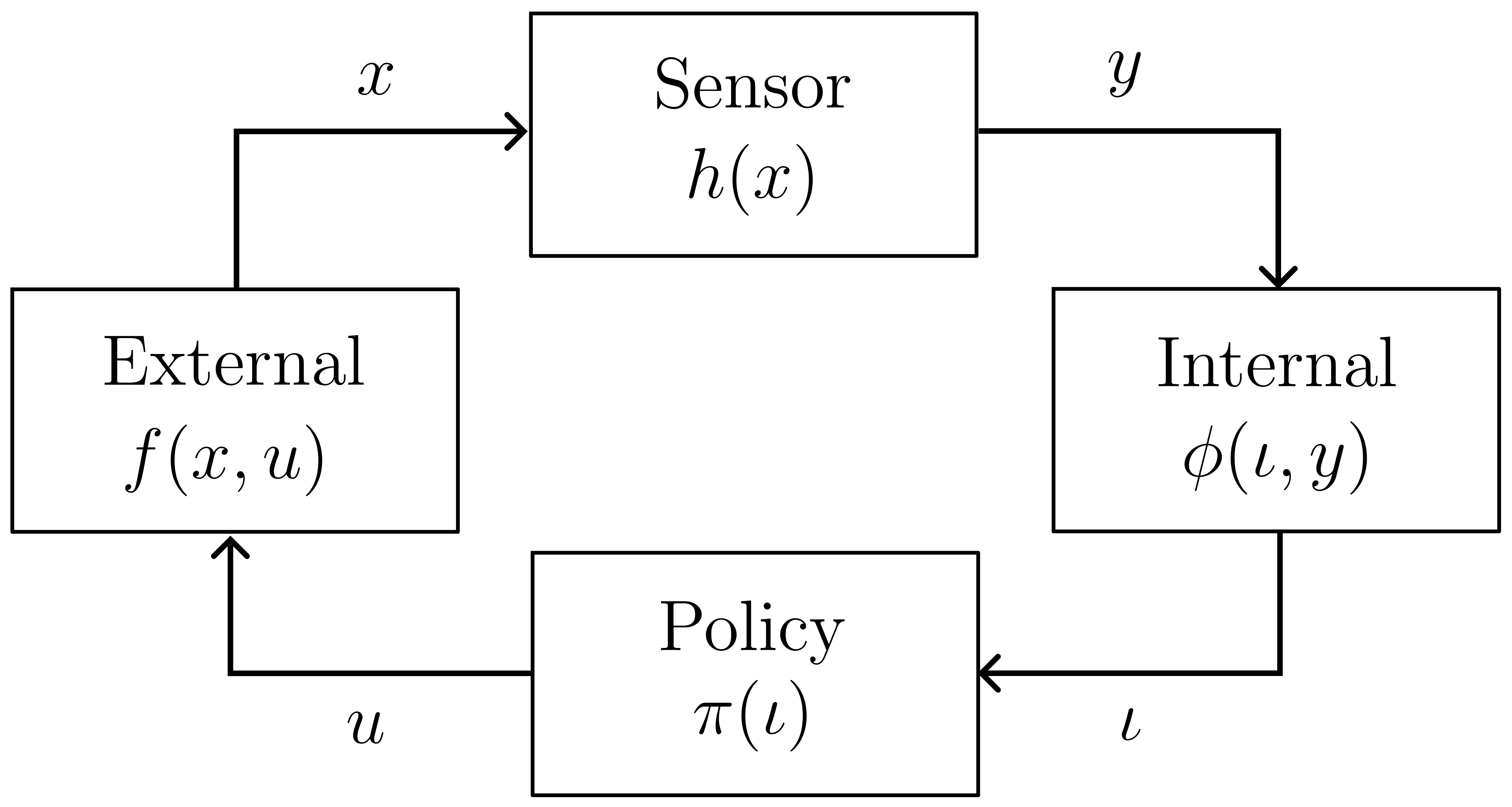} \\
  (a) & & (b)
\end{tabular}
\end{center}
\caption{\label{fig:extint}(a) The {\em internal} robot brain is
  defined as an ITS that interacts with the {\em external} world
  (robot body and environment). (b) Coupled internal and external
  systems mathematically capture sensing, actuation, internal
  computation, and the external world.}
\end{figure*}

\subsection{Previous work}

Many of the concepts in this paper build upon \citep{WeiSakLav22}, in
which we recently proposed an enactivist-oriented model of cognition
based on information spaces. By {\em enactivist} \citep{HutMyi12}, it
is meant that the necessary brain structures emerge from sensorimotor
interaction, and do not necessarily have predetermined, meaningful
components as in the classical representationalist paradigm
\citep{NewSim72}. 
Despite introducing a general framework, the focus
of \citep{WeiSakLav22} was on emerging equivalence classes of the
environment states through sensorimotor interactions. In some sense,
this corresponds to finding a (minimal) sufficient sensor or a
representation for the external system. In this paper, we focus on the internal system which, in turn, refers to finding a minimal filter
and/or a policy, \BS{given an appropriate task description}.

Most approaches in filtering can be categorized into two classes: probabilistic and combinatorial. Probabilistic filters typically rely on Bayes' rule to propagate the obtained information \citep{ThrBurFox05,Sar13}. They have been extensively used in robotics; especially for state estimation, mapping, and localization \citep{DisNewClaDurCso01,ZheZenSob17}. Combinatorial filters \citep{TovCohLav08,KriShe12} do not rely on probabilistic models but instead make use of nondeterministic (possibilistic) ones. The notion of information spaces \citep[Chapter 11]{Lav06} provides a general formalization that encompasses both probabilistic and combinatorial filters. In our framework based on transition systems, the set of states of a transition system will indeed be an information space. Using the notion of information spaces, several works attempted to characterize different sensors defined over the same state space and compare their power in terms of gathered information \citep{OkaLav07b, Lav12b, ZhaShe21}. Despite providing an elegant and exact way of solving filtering problems, the space requirements for a combinatorial filter can be high. Considering a notion of minimality, some authors addressed algorithmic reduction of filters \citep{RahOka21,OkaShe17,SonOka12}. 
Reducing combinatorial filters have roots in the theory of computation, where decomposition \citep{Har60, HarSte64} and minimization of finite automata (Moore and Mealy machines) has been a topic of active research since the 1950s.

A gap still remains between analyzing the requirements of filters or sensors for pure inference (passive filtering) and the ones needed for active tasks (planning and control) such that an \emph{information-feedback} policy can be described. Various representations were used in the literature as a domain to define the policy.
For most robotic planning problems, the domain of the policy or a plan is fixed; which is the robot configuration or the phase space \citep{MajTed17,ZhuAlo19}. This corresponds to the assumption that the robot state can be fully observed or estimated with high accuracy. For problems that the state is not fully observable, \ac{POMDPs} \citep{KaeLitCas98,RosPinPaqCha08} and \emph{belief spaces} \citep{VitTom11,AghChaAma14} have been considered for planning.  Note that POMDP literature is mostly restricted to finite state and action spaces.
There is a limited literature that studied the information requirements for active tasks which corresponds to determining the weakest notion of sensing or filtering that is sufficient to accomplish a task. A notable early work showed, especially for manipulation, that one can achieve certain tasks even in the absence of sensory observations \citep{ErdMas88}. In \citep{ZhaShe20} the authors characterize all possible sensor abstractions that are sufficient to solve a planning problem. Minimality has been addressed for specific problems regarding mobile robot navigation in \citep{BluKoz78,TovMurLav07}.
Closely related to our work, a language-theoretic formulation appears in \citep{SabGhaSheOka19}, in which, Procrustean-graphs (p-graphs) were proposed as an abstraction to reason about interactions between a robot and its environment. 

Obtaining a model from input-output relations that represents the
underlying system has been a common interest for many fields ranging
from control theory, machine learning, and robotics. Different
approaches to this problem in the context of finite state automata
were reviewed by \citep{Pit89}. In diversity-based inference (DBI) for
an input-output machine \citep{Bai77, RivSch94, RivSch93}, a model of
the underlying system is constructed in terms of equivalence classes
of \emph{tests} which consist of one or more consecutive actions and
observations. Its probabilistic counterpart, predictive state
representation (PSR) \citep{LitSut01, BooSidGor11}, addresses the
analogous problem by considering (linear) combinations of prediction
vectors, which represent probabilities for test success/failure. Other than relying solely on input-output relations\new{,} a parametric model (or a class of them) can be provided for learning a representation for the underlying system \citep{BruBerBraLecHasRusGro22}. We will also consider this distinction between model-based and model-free within our ITS framework.

\subsection{Contributions}

The main contribution of this paper is a novel mathematical framework for analyzing and distinguishing the interactions that emerge from a robotic system embedded in an environment. We introduce the notion of ITSs as a general way to characterize the internal system (``brain") of the robot.  We then proceed to establish conditions for sufficiency and the existence of unique minimal ITSs in a very general setting.  Intuitively, we establish how small the robot brain could possibly be for given goals or tasks.  Anything less results in impossibility.  The framework addresses both filtering and sensor fusion problems, which are passive in the sense of no controls are applied, and planning or control problems, which are active.  We illustrate the \newnew{scope} 
of the framework by applying it to several problems that shed light on relationships to many existing concepts, including Kalman filters, predictive state representations, combinatorial filters, and planning over reduced information spaces.

\BS{Many of the concepts in this paper build upon \citep{WeiSakLav22} and \cite[Chapter 11]{Lav06}. The contributions of the present paper with respect to these works are listed in the following:
\begin{itemize}
    \item The framework based on transition systems introduced in \citep{WeiSakLav22} is adapted into a robotic setting. We define the notions of internal and external systems within this context, together with concrete examples. Moreover, we formulate the disturbances affecting the external system model and the sensor mapping, within this framework. 
    \item We formally define the notion of task description (distinguishing between finite and infinitary tasks, as well as between active and passive tasks), and filtering and planning problems.
    \item The notion of minimality for a transition system describing the internal system under a planning (control) task is new. 
    \item The ITSs corresponding to some of the information spaces presented in \cite[Chapter 11]{Lav06}, which were based on intuitions, are shown to be minimal applying the proposed framework.
    \item We also formalize the model-based and the model-free information spaces using the notion of coupled internal-external systems.
\end{itemize}
}
This paper is an expanded version of \citep{SakWeiLav22}. 

\subsection{Paper structure}
The remaining of the paper is organized as follows.
Section \ref{sec:math} provides a general mathematical formulation of
robot-environment interaction as transition systems. We also introduce a notion of couplings which encode various types of interactions between an internal and an external system.
Section~\ref{sec:suff} then develops central notions of sufficiency and minimality over the space of possible ITSs. 
Section~\ref{sec:problems} applies the general concepts to address what it means to solve both passive (filtering) and active (planning/control) tasks minimally. 
Canonical problem families are presented that aim to capture typical problem settings in filtering and planning, from the perspective of minimal sufficient solutions.
Section~\ref{sec:ex} illustrates how the theory can be applied using
simple examples. Section~\ref{sec:discussion} summarizes the contributions and identifies important directions for future work.

\section{Mathematical Models of Robot-Environment Systems}
\label{sec:math}

\subsection{Internal and external systems}
\label{ssec:IntExt}

In this paper, we consider a robot embedded in an environment and
describe this system as two subsystems, named \emph{internal} and
\emph{external}, connected through symmetric input-output
relations.
\newnew{The external system describes the physical world, and the internal system describes the information processing ``robot brain''.}
\BS{With ``robot brain" we refer to a centralized computational component that processes sensor observations and actions.} The interaction between the internal and external systems is shown in Figure~\ref{fig:extint}(b).
The input to the internal system is the information reflecting the state of the external system, obtained through observations (that is the output of the external system).
The output of the internal system is a
control command that in turn corresponds to the control input \newnew{of} 
the external system, and causes its state to change.  In this sense, the
state of the external system is similar to the use of the term \emph{state} in control
theory and the state of the internal system is similar to the use of the term in computer science. 

The external system corresponds to the totality of the physical environment, including the robot body. Let $X$ denote the set of states of this system; it could be for example, \BS{the configuration of a robot (e.g., position and orientation of a mobile robot or joint configuration of a robotic manipulator) within a known environment (or within a set of possible environments), or it can be extended to include also the (higher order) derivatives of its configuration.}
See \citep[\newnew{Section} 3.1]{Lav12b} for possible state spaces of a mobile robot. Next,
let $U$ be the set of control inputs (also referred to as actions).
When applied at a state $x \in X$, a control $u \in U$ causes $x$ to
change according to a state transition function
$f\colon X \times U \to X$. Indeed, an action $u\in U$ refers to the
control input to the system and corresponds to the stimuli created by
a control command generated by a decision maker. Mathematically,
\emph{the external system} can now be expressed as the
triple~$(X,U,f)$. The sets $X$ and $U$ can be finite or infinite
discrete spaces, or they could be equipped with extra structure: they
could be metric manifolds, vector spaces, compact, or non-compact
topological spaces. In such cases, the function $f$ may or may not be
assumed to respect such structure: sometimes it is appropriate to
assume continuity or measurability.

The internal system (robot brain) corresponds to the perspective of
a decision maker. The states of this system correspond to the retained information gathered through the outcomes of actions in terms of sensor observations. To this end, the basis of our mathematical formulation of the internal system is the notion of an
\emph{\ac{Ispace}} presented in \citep[Chapter 11]{Lav06}. Let $\Ifs$ be the set of these internal information states. We will use the term \emph{\ac{Istate}} to refer to elements of $\Ifs$ and denote them by~$\is$. 
Similar to the external system, the internal system evolves with each $y \in Y$ according to the \emph{information transition function} $\phi\colon \Ifs \times Y \rightarrow \Ifs$.  The \emph{internal system} can now mathematically be described by the
triple~$(\Ifs,Y,\phi)$.

The external $(X,U,f)$ and internal $(\Ifs,Y,\phi)$ can be
\emph{coupled} to each other to create a \emph{coupled dynamical system}. This is achieved by introducing two coupling functions,
that match outputs of one system to the inputs of the other and vice
versa. For us, these are the \emph{sensor mapping} $h\colon X\to Y$
and the \emph{policy} $\pi\colon \I\to U$, see
Figure~\ref{fig:extint}. The function $h$ labels the external states
with sensory data, and $\pi$ labels the internal information states
with the actions.
Note that $\pi$ can be seen as an information feedback policy sending a control to the external. 
\BS{In the following parts of this paper, we will refer to $\pi$ simply as the policy. Therefore, it will be a map from the states of an \ac{Ispace} (Sections~\ref{sec:hist_ispace} and \ref{sec:derived_ITS} will present possible \ac{Ispace} descriptions) to the set of controls. This definition is more general than the use of policy in the robotics literature which typically refers to a state-feedback policy, that is, a map from the states of a deterministic or a probabilistic description of the external system.}

\BS{Suppose the system evolves in discrete stages.} 
\BS{Then}, the coupled dynamical system can be written as
\begin{align}
  \centering
  \is'&=\phi(\is,y) &\text{in which } y & = h(x), \nonumber \\
  x'  &=f(x, u)     &\text{in which } u & = \pi(\is'). \label{eqn:coupled_sys}
\end{align}
Here we use $x'$ to refer to the next state, not the derivative
of~$x$. Whereas the equations on the left side describe the evolution
of this coupled system, the ones on the right show the respective
outputs of each subsystem. The coupled system of internal and external
described this way is an autonomous system, meaning that given an
initial state $(\is, x) \in \Ifs \times X$ there exists a unique
state trajectory.%
\endnote{In dynamical systems terminology, the coupled system is a closed, deterministic dynamical system. In the presence of extrinsic disturbances (see Section~\ref{ssec:Disturbances}) the coupled system becomes a non-autonomous dynamical system.}
We denote the function $(\is,x)\mapsto (\is',x')$ by
$\phi*_{\pi,h}f$ which highlights that $\phi*_{\pi,h}f$ is a
coupling of $\phi$ and $f$ via
the pair of coupling functions
$(\pi,h)$. Then, the coupled system is the pair
$$(\Ifs\times X,\,\phi*_{\pi,h}f).$$

For the external
system, starting from an initial state $x_1$, each stage $k$
corresponds to applying an action $u_{k}$ which then yields the next
stage $k+1$ and the next state $x_{k+1}=f(x_{k},u_{k})$.  As the
system evolves through stages, the tuples $\histx_k=(x_1,x_2,\dots,x_k)$,
$\histu_{k-1}=(u_1,u_2,\dots,u_{k-1})$ and
$\histy_k=(y_1,y_2,\dots, y_k)$ correspond respectively to the state, action, and observation
histories up to stage $k$, with $y_i=h(x_i)$ for
$i\in\{1,\dots,k\}$.
Note that applying the action
$u_k$ at stage $k$ would result in a transition to state $x_{k+1}$ and
the corresponding sensor reading $y_{k+1}=h(x_{k+1})$.  The same
applies for the internal system. We can describe its evolution
starting from an initial \ac{Istate} $\is_0$, and following the state
transition equation $\is_{k}=\phi(\is_{k-1},y_k)$. At stage $k$,
$\pi(\is_k)$ produces the action $u_k$. Note that the stage index of
the \ac{Istate} starts from~$0$. In some cases, $\is_0$ can encode prior information regarding the external system and in others, it does not. 
We will consider this distinction more formally in Section~\ref{sec:hist_ispace}.  The next
information state $\is_1$ is obtained using $\is_0$
and~$y_1$. 
\BS{We assume that no control command (action) is outputted at stage $0$, meaning that the control history starts with $u_1$. By convention, $\histu_0=()$ is an empty sequence.}

\begin{figure}
  \centering
  \includegraphics[width=.9\linewidth]{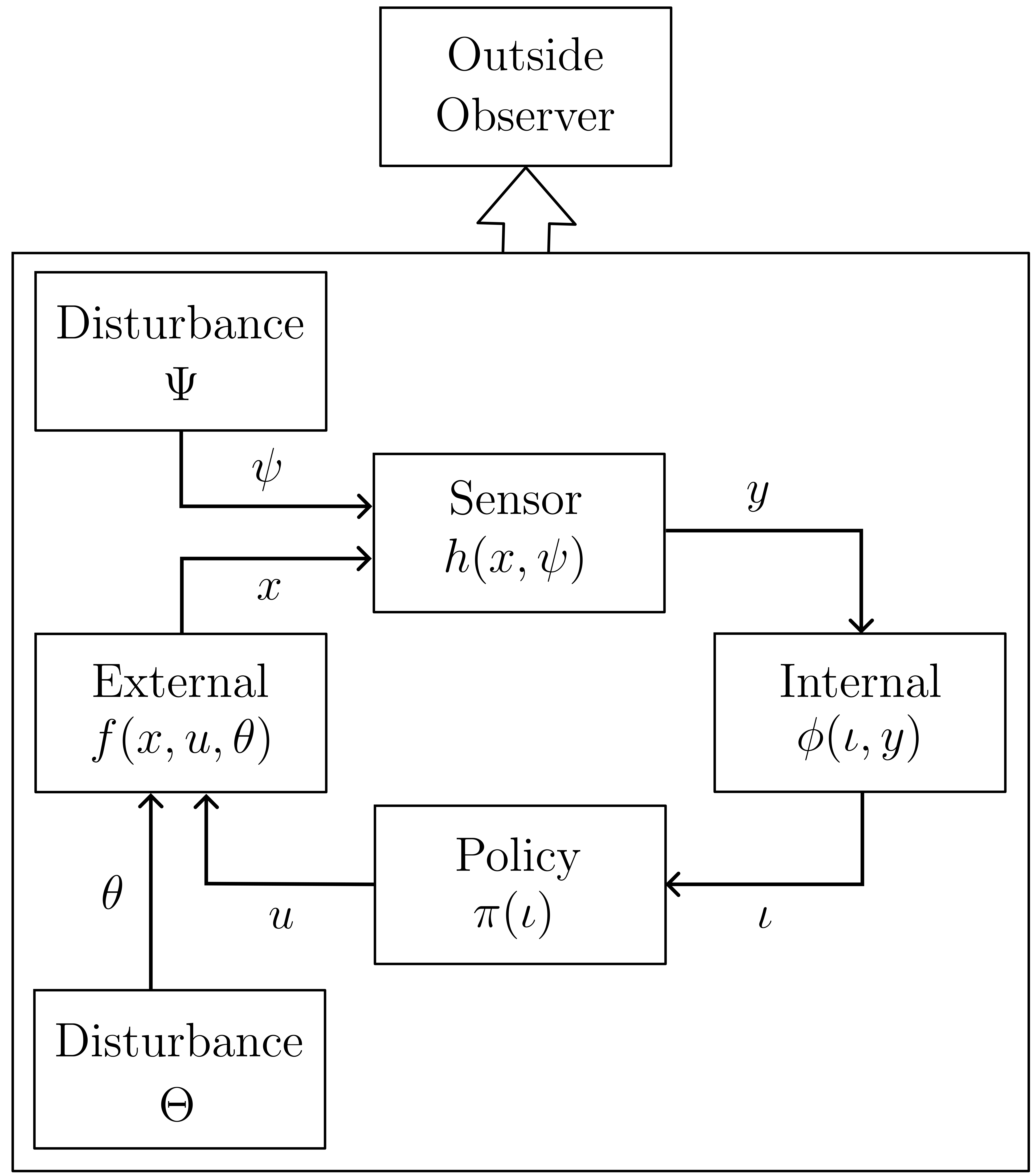} 
  \caption{\label{fig:extint_dist} Disturbances may affect the external system and the sensor. Note that conditioned on the realization of these, the internal system and policy remain deterministic. An outside observer (planner / designer) may perceive the coupled system as a whole.}
\end{figure}

\subsection{Disturbances} \label{ssec:Disturbances}

The coupled internal-external systems formulation can be extended to
include disturbances affecting the external system and the sensor. In
particular, we can define two disturbance generating systems, with
outputs $\theta \in \Theta$ and $\psi \in \Psi$, that are influencing
the external system and the sensor, respectively (see
Figure~\ref{fig:extint_dist}). Mathematically, the external system
with disturbances is $(X,U\times \Theta,f)$, where
$f\colon X\times (U \times \Theta) \to X$ is a state transition function
for the external system under disturbances. Thus, the disturbances
merely add a new dimension to the control parameters of the system.
In the internal-external coupling, we also assume disturbances in the sensory mapping which
takes the form $h\colon X\times \Psi\to\Ifs$. Then, the definition of the
coupled internal-external system given in~\eqref{eqn:coupled_sys} is
modified as follows:
\begin{align}
  \centering
  \is'&=\phi(\is,y) &\text{in which } y &= h(x, \psi), \nonumber \\
  x'&=f(x, u, \theta) &\text{in which } u &= \pi(\is') .\label{eqn:coupled_sys_dist}
\end{align}
Here, the other two functions $\phi$ (the information transition
function) and $\pi$ (policy) are as in Section~\ref{ssec:IntExt}.

Note that neither $\theta$, which affects the state transition function
of the external system, nor $\psi$, which affects the sensor mapping, is
directly available to the internal system $(\Ifs,Y,\phi)$, which is
just as in Section~\ref{ssec:IntExt}.  However, some information
regarding the disturbances can be specified for an internal system that
makes use of a model of the external. \newnew{These} could be encoded into
the set $\Ifs$ and the transition function $\phi$. 
Then, the internal system has a
non-trivial correlation with the disturbance, even though it is never
directly perceived.
We will consider the distinction between model-based and model-free systems in Section~\ref{Section_Model_based_model_free}. Finally, the coupled system is mathematically a triple $(\Ifs \times X,\Theta\times\Psi,g)$ where $g$ is a function that, given a state $(\iota, x) \in \Ifs \times X$ and a disturbance parameter $(\theta,\psi)\in\Theta\times\Psi$, outputs the next state in $ \Ifs \times X$, formally, $$g\colon (\Ifs \times X ) \times (\Theta \times \Psi) \to \Ifs \times X.$$

There are two possibilities for how the information
regarding the disturbances can be specified: nondeterministic and
probabilistic. In the nondeterministic case, the set $\Psi$ and
possibly a subset 
$\Psi(x)\subseteq \Psi$ is specified for all $x\in X$, in which
$\Psi(x)$ represents the set of all $\psi$ that can be realized for
each $x$. In the probabilistic case, assuming that the disturbances
are generated by a system that is Markovian, such that they do not
depend on the previous stages, a probability distribution over $\Psi$
can be specified for each $x$. This will be denoted as
$P(\psi \mid x)$.  The disturbances affecting the external system can
be specified similarly to those that affect the sensor. In the
nondeterministic case, the set $\Theta(x,u) \in \Theta$ is known for
each $(x,u) \in X\times U$. In the probabilistic case, a probability
distribution over $\Theta$, that is, $P(\theta \mid x,u)$, can be
specified for each $(x,u)$.

\subsection{Generalizing to transition systems}\label{sec:trans_sys}

A \emph{transition system} is a triple $(S, \Lambda, T)$, in which, $S$ and
$\Lambda$ are some sets (possibly equipped with some structure, for example,
topology), and $T \subseteq S\times\Lambda\times S$ is a ternary relation.
Here $S$ is the set of states, $\Lambda$ is the set of labels for
transitions between elements of $S$, and a triple $(s,\lambda,s')$ belongs to $T$ if there is a transition from $s$ to $s'$ labeled by~$\lambda$. A special case is when for each $(s,\lambda)\in S\times \Lambda$ there is a unique $s'\in S$ with $(s,\lambda,s')\in T$. Then, $T$ defines a function $\tau\colon S\times\Lambda\to S$.  These are called \emph{deterministic transition systems}, and sometimes also \emph{(open) dynamical systems}. An extensive analysis of those and their coupling is explored by~\citep{Spi15}.

All models in Sections \ref{ssec:IntExt} and~\ref{ssec:Disturbances} are
deterministic transition systems: the external, the internal, the
disturbed versions, and their couplings are all deterministic
transition systems.\endnote{Although the values of the disturbances, when present, are not known beforehand or deterministically, once the disturbed parts of the coupled system receive the value of the disturbance, they produce a uniquely determined output. Hence, these components and their couplings are deterministic transition systems.} Note that if $\Lambda$ is a singleton, the system
$(S,\Lambda,\tau)$ is equivalent to a discrete time autonomous dynamical
system.
If $(S,\Lambda,\tau)$ is a deterministic transition system, 
\BS{$S$ and $\Lambda$ are finite,}
$s_0\in S$, and $F\subseteq S$, then $(S,\Lambda,\tau,s_0,F)$
is a finite automaton as defined in \cite[Definition~1.5]{Sip12}. If not stated otherwise, we do not assume our systems to be finite. In Section~\ref{ssec:DBI} we explore connections between our theory with the theory of finite systems.

The following notion of state-relabeled transition systems was
introduced in~\citep{WeiSakLav22} to model the internal and external systems.

\begin{definition}[\textbf{State-relabeled transition system}] 
  A state-relabeled transition system is the quintuple
  $(S,\Lambda,T,\sigma,L)$ in which $\sigma\colon S \rightarrow L$ is a
  labeling function and $(S,\Lambda,T)$ is a transition system.
  The function $\sigma$ is the \emph{labeling function} and
  $L$ is the set of labels.
\end{definition}

A state-relabeled transition system is closely related to
the \emph{Moore machine} which is a state-relabeled transition system with a fixed initial state, \BS{a finite set of states, and finite sets of input and output alphabets (finite $\Lambda$ and $L$)~\citep{Moore56}.}

\new{In our framework, a labeling function $\sigma$ serves two purposes; it enables a potential coupling by matching output of one system to the input of another, and acts as a categorization of the states of the system being labeled.}
Preimages of a labeling function $\sigma$ induce a partition of the
state space $S$ into sets whose elements are indistinguishable through
sensing. Let $S/\sigma$ be the set of equivalence classes $[s]_\sigma$
induced by $\sigma$ such that $S/\sigma=\{[s]_\sigma \mid s\in S\}$
and $[s]_\sigma=\{s'\in S \mid \sigma(s')=\sigma(s)\}$.  Then, using
these equivalence classes, we can define a new transition system called
the \emph{quotient} of $(S,\Lambda,T)$ by~$\sigma$.

\begin{definition}[\textbf{Quotient system}] \label{Def_Quotient_of_TS}
  The quotient of $(S,\Lambda,T)$ by $\sigma$ is the transition system
  $(S/\sigma,\Lambda,T/\sigma)$, in which
  \begin{equation*}
    T/\sigma :=\big\{\big([s]_\sigma, \lambda, [s']_\sigma\big) \mid (s, \lambda, s') \in T \big\}.
  \end{equation*}
\end{definition}

Note that $(S/\sigma,\Lambda,T/\sigma)$ is a reduced version of
$(S,\Lambda,T)$, in the sense that the map $s\mapsto [s]_\sigma$ is
onto, but not necessarily one-to-one.%
\endnote{Considering homomorphisms between transition systems goes
  beyond the present paper; a good source for the general theory
  is~\citep{GORANKO2007249}.}
We might be interested in finding a labeling function $\sigma$ such
that the corresponding quotient transition system is as simple as
possible while ensuring that it is still useful. In the following
sections, we will provide motivations for a reduction and discuss in
more detail the requirements on $\sigma$ for the quotient system to be useful.

The external and internal systems can be written as 
state-relabeled deterministic transition systems $(X,U,f,h,Y)$ and
$(\Ifs,Y,\phi,\pi,U)$, respectively, in which $h$ and $\pi$ are
considered as labeling functions.
Interpreting the labels as the output of a transition system, a coupled internal-external system can be described in terms of the
state-relabeled transition systems formulation too, so that the output of one transition system is an input for another.  Described this way, a coupling of two transition systems results in unique paths in either transition system, initialized at a particular state.

\section{Sufficient Information Transition Systems}\label{sec:suff}

\subsection{Information transition systems}\label{sec:ITS_prespectives}

In the general setting, an \ac{Istate} corresponds to the available
(stored) information at a certain stage with respect to the action and observation histories.  An \ac{Ispace} is a collection of all possible \ac{Istate}s.  We will use the acronym ITS to refer to an \emph{information transition system}, that is, a transition system whose state space is an \ac{Ispace}.

We have already used the notion of an \ac{Ispace} when modeling the
internal system representing the robot brain, which we view as an \ac{ITS}.
Here, we extend the notion of an \ac{ITS} to include different perspectives from which the external and the coupled systems
can be viewed. In particular, we identify three perspectives;
\begin{itemize}
\item a planner,
\item a plan executor, 
\item and an (independent) observer.
\end{itemize}
With a slight abuse of previously introduced notation and terminology, we use the term \emph{internal} to refer to any system that is not
the external system and we use $\Ifs$ to denote a generic \ac{Ispace}.  We use the term \ac{DITS} to refer to an \ac{ITS} for which the transitions are governed by an information transition function so that they are deterministic.  We denote these types of systems \newnew{by} 
$(\Ifs, \Lambda, \phi)$, in which $\Lambda$ is the edge \newnew{labeling} 
and $\phi\colon\Ifs \times \Lambda \to \Ifs$ is an information transition function. Otherwise, an \ac{ITS} will be called a \ac{NITS} \newnew{and denoted by} 
\new{$(\Ifs, \Lambda, \Phi)$}, 
in which $\Phi \subseteq \Ifs \times \Lambda \times \Ifs$ is the transition relation.

Suppose $(\Ifs,U\times Y,\Phi)$ is a \ac{NITS} and $\pi\colon\I\to U$ a policy, \new{and define}
\begin{equation}\label{eq:restNITS}
  \Phi_\pi := \{ (\is, (u,y), \is') \in \Phi\mid u=\pi(\is) \} \, \new{\subseteq \Phi}.
\end{equation}
\new{The transition system $(\Ifs,U\times Y,\Phi_\pi)$ is called the \emph{restriction} of $(\Ifs,U\times Y,\Phi)$ by the policy~$\pi$.}
\endnote{Note that we can treat the external system symmetrically and given a non-deterministic transition system $(X,Y\times U,F)$ with $F\subseteq X\times (Y\times U) \times X$ and a sensor mapping $h\colon X\to Y$, we can define the restriction $F_h$ in an analogous way as above.  However, because it is not central to this paper we will not elaborate on this topic.}
If $(\Ifs,U\times Y,\phi)$ is a \ac{DITS}, the \emph{strong restriction} by $\pi\colon \Ifs\to U$ is given by $(\Ifs,Y,\phi_\pi)$,
in which $\phi_\pi\colon \Ifs\times Y\to\Ifs$ and 
$\phi_\pi(\is,y)=\phi(\is,\pi(\is),y)$. The strong restriction is obtained by first taking the restriction of $\phi$ treated as a subset of $\Ifs\times (U\times Y) \times\Ifs$ and then taking the projection of the resulting set onto \mbox{$\Ifs\times Y\times\Ifs$}.

Before any policy is fixed, a \ac{DITS} of the form $(\Ifs,U\times Y,\phi)$ corresponds to the \emph{planner} perspective. Once the policy is fixed, the strong restriction $(\Ifs,Y,\phi_\pi)$, which is just as the internal system was defined in Section~\ref{sec:math}, corresponds to the \emph{plan executor}.

\begin{example}[\textbf{\emph{A binary toy model}}]
\new{
Consider the DITS $(\ifs, U\times Y, \phi)$ which corresponds to a planner perspective. Suppose $U=Y=\Ifs=\{0,1\}$ and let $\phi\colon \Ifs\times (U\times Y) \to \Ifs$ be defined by $\phi(\is,(u,y))=|y-u|$. Suppose a policy $\pi: \ifs \rightarrow U$ is fixed such that $\pi(\is)=\is$. Then, $(\ifs, Y, \phi_\pi)$, in which $\phi_\pi(\is, y)=|y-\pi(\is)|$, is the strong restriction of $(\ifs, U\times Y, \phi)$ by $\pi$. Furthermore, it corresponds to the plan executor.}
\end{example}

In this paper, an observation will refer to a sensor reading~$y$.  However, when we discuss an (independent) observer described over the coupled system, the input to this observer system can be a function of any variable of the coupled system, for instance action, information state or the state of the external. If the coupled system is disturbed, the disturbances can be observed by the observer too.

\subsection{History information spaces}\label{sec:hist_ispace}

\newnew{The most fundamental \ac{Ispace} is the \emph{history \ac{Ispace}}, which we denote by $\Ifshist$. A \emph{history \ac{Istate}} at stage $k$ corresponds to all the information that is gathered through sensing (and potentially also through actions) up to stage $k$, assuming perfect memory. In this sense, $\Ifshist$ is the canonical \ac{Ispace}, and all the other \ac{Ispace}s are derived from it.  We denote the history \ac{Istate}s by the letter $\his$ to distinguish them from the states of other information spaces, which 
we typically denote by $\is$ (recall the notation introduced in Section~\ref{ssec:IntExt}).}

Let $U$ and $Y$ be the sets of possible actions and observations respectively.  The elements of $\Ifshist$ are finite sequences of alternating actions and observations which build upon some initial state $\eta_0\in \Ifshist$.
Denote the set of possible initial states of $\Ifshist$ by $\Ifs_0$. Then the elements of $\Ifshist$ are of the form 
\begin{equation}
  \label{eq:SteveNotation1}
  (\his_0,\histu_{k-1}, \histy_k):=(\his_0,y_1,u_1,y_2,u_2\dots,u_{k-1},y_k)
\end{equation}
for $k\in\N$, in which $\his_0\in \Ifs_0$, $u_i\in U$ and $y_i\in Y$ for all $i\le k$. Additionally, denote
\begin{equation}
  \label{eq:SteveNotation2}
  \his_k:=(\his_0,\histu_{k-1},\histy_k).
\end{equation}
The notations \eqref{eq:SteveNotation1} and \eqref{eq:SteveNotation2} follow \citep[Chapter 11]{Lav06}. The lower index $k$ refers to the
\emph{stage} of the state, or the length of the action-observation sequence.  The convention here, as already mentioned in the end of Section~\ref{ssec:IntExt}, is that $\histu_0$ is assumed to be the null-tuple. Thus, $\his_k=(\his_0, \histu_{k-1}, \histy_k)$ is the \ac{Istate} at stage $k$, which is achieved by \BS{iteratively concatenating the action-observation pairs $(u_{i-1},y_{i})$ at the end of the sequence for $i\in \{1,\dots,k\}$ after the initial state~$\his_0$.}

The description of initial conditions in the set $\Ifs_0$ varies with the available prior information. We discuss these descriptions
below. The history information space at stage $k$ is the subset of $\Ifshist$ which consists of elements of the form given by
\eqref{eq:SteveNotation1} for fixed $k$, and can be expressed as the product
\begin{equation} \label{Eq_Def_finite_Ifshist}
\Ifs_k := \Ifs_0 \times \histU_{k-1} \times \histY_k,
\end{equation}
in which
$\histU_{k-1}=U^{k-1}$ and $\histY_k=Y^{k}$. In general, the number of stages that the system will go through is not fixed. Therefore we
assume the history \ac{Ispace} to contain all finite action-observation sequences, that is, $\Ifshist= \bigcup_{k \in \mathbb{N}}\Ifs_k$.  The \ac{DITS} corresponding to $\Ifshist$ is $(\Ifshist, U \times Y, \fmaphist)$, in which 
$$\fmaphist(\his,u,y)=\his\cat (u, y),$$ 
and $\cat$ is the concatenation of two sequences.  Note that the concatenation operation makes $(\Ifshist,\cat)$ into a free monoid.
The derived information transitions systems which will be introduced in Section~\ref{sec:derived_ITS} can be seen as quotients of this monoid by equivalence relations; sometimes these quotients can also be monoids, or even groups.

\subsection{Sufficient state-relabeling}
\label{ssec:suff}
In \citep{WeiSakLav22} we have introduced a notion of \emph{sufficiency} that generalizes the definition introduced in \citep[Chapter 11]{Lav06} and is presented here for completeness.

\begin{definition}[\textbf{Sufficient labeling function}]
\label{def:sufficiency}
Let $(S,\Lambda,T)$ be a transition system. A labeling function $\sigma\colon S \rightarrow L$ defined over the states of a transition system is \emph{sufficient} if and only if for all $s,t,s',t' \in S$ and all $\lambda \in \Lambda$, the following implication holds:
  \begin{multline*}
    \sigma(s)=\sigma(t) \land (s,\lambda,s')\in T \land (t,\lambda,t')\in T \implies \\
    \sigma(s')=\sigma(t'). 
  \end{multline*}
If $\sigma$ is defined over the states of \new{a deterministic transition system}
$(S,\Lambda,\tau)$, then $\sigma$ is sufficient if and only if for all $s,t\in S$ and all $\lambda \in \Lambda$, $\sigma(s)=\sigma(t)$ implies that $\sigma(\tau(s,\lambda))=\sigma(\tau(t,\lambda))$.
\end{definition}

Consider the stage-based evolution of
the state-relabeled deterministic transition system corresponding to the
external system $(X,U,f,h,Y)$ with respect to the action (control
input) sequence $\histu_{k-1}=(u_1,\dots,u_{k-1})$. This corresponds
to the state and observation histories till stage $k$, that are
$\histx_k=(x_1,\dots,x_k)$ and $\histy_k=(y_1,\dots,y_k)$. Recall that
applying $u_k$ at stage $k$ would result in a transition to $x_{k+1}$
and the corresponding observation $y_{k+1}=h(x_{k+1})$.
Hence, in this context, sufficiency of $h$ implies that given the
label $y_k=h(x_{k})$ and the action $u_k$, it is possible to determine
the label $y_{k+1}=h(x_{k+1})$. One interpretation of sufficiency of
$h$ is that the respective quotient system 
sufficiently represents the underlying system up to the equivalence
classes induced by~$h$. This notion is similar to a minimal realization
of a system, that is, the minimal state space description that models
the given input-output measurements (see for example
\citep{KotMooTon18}). \BS{Another} interpretation is in a predictive sense.
Suppose the quotient system is known. Then, the label
$y_{k+1}=h(x_{k+1})$ can be determined before the system gets to
$x_{k+1}$, using the current label $y_k$ and the action to be applied
$u_k$.  Furthermore, under a fixed policy, the complete
observation trajectory can be determined from the initial observation
by induction.

Now, consider an internal system with a labeling function
$\imap\colon \Ifs \rightarrow \Ifs'$, that is,
$(\Ifs,U\times Y,\phi,\kappa,\Ifs')$, and its evolution with respect
to the histories $\tilde{y}=(y_1,\dots,y_{k})$ and
$\tilde{u}=(u_1,\dots,u_{k-1})$. At stage $k$, the state of the
\ac{DITS} is $\is_{k}$ and with $(u_k,y_{k+1})$ the system transitions
to $\is_{k+1} = \phi(\is_k, u_k, y_{k+1})$. Sufficiency of $\imap$
implies that given $\imap(\is_k)$, $u_k$, and $y_{k+1}$, we can
determine $\imap(\is_{k+1})$. This is equivalent to the definition
introduced in \citep[Chapter 11]{Lav06} and makes it a special case of
Definition~\ref{def:sufficiency}.

\subsection{Derived information transition systems}\label{sec:derived_ITS} 

Even though it seems natural to rely on a history \ac{ITS}, the dimension of a history \ac{Istate} increases linearly, and the size of the history \ac{Ispace} increases exponentially, as a function of the stage index, making it impractical in most cases. Thus, we are interested in defining a reduced \ac{ITS} that is more manageable, due to for example lowered requirements for memory or computing power. Furthermore, this would largely simplify the description of a policy for a planner or a plan executor.

Recall the quotient of a transition system by a labeling function (see Definition~\ref{Def_Quotient_of_TS}).
We rewrite $(\Ifshist, U \times Y, \phi_{hist})$ as $(\Ifshist, U \times Y, \Phi_{hist})$, in which
\begin{multline} \label{eq:Set_Phi}
  \Phi_{hist}= \Big\{ \big(\his, (u,y), \his' \big) \in\\
  \Ifshist \times (U \times Y) \times \Ifshist \mid \his' = \phi_{hist}(\his,u,y) \Big\}.   
\end{multline}
We can introduce an \ac{Imap} $\imap\colon \Ifshist \rightarrow \Ifsder$ that categorizes the states of $\Ifshist$ into equivalence classes
through its preimages.  In this case, $\imap$ serves as a labeling function. A reduction is obtained in terms of the quotient of $(\Ifshist, U \times Y, \Phi)$ by $\imap$, that is, $(\Ifshist/\imap, U \times Y, \Phi/\imap)$ as histories are grouped into equivalence classes.

It is crucial that the derived \ac{ITS} is a \ac{DITS} so that the transition from the current label to the next can be determined using only the derived \ac{ITS}, without making reference to the history~\ac{ITS}. The reason for this requirement is straightforward for an observer as the \ac{Istate}s correspond to what is inferred about the external system, given observation history (potentially accompanied
by the action history). The same applies for the planner and the plan executor to be able to describe and execute a policy. Considering the quotient system derived by $\imap$ from the \ac{DITS} (by definition) $(\Ifshist, U \times Y, \phi)$, we cannot always guarantee that the resulting \ac{ITS} is deterministic. This depends on the \ac{Imap} used for state-relabeling, as illustrated in the following proposition.

\begin{proposition}
[\textbf{Quotient of a history ITS may be a NITS}]
  \label{prop:non_sufficient_label}
  For all non-empty $U$ and $Y$, and for the corresponding $\Ifshist$,
  there exists a labeling function $\imap$ such that the quotient $(\Ifshist/\imap, U \times Y, \Phi/\imap)$ of
  $(\Ifshist, U \times Y, \phi)$ by $\kappa$,
  in which $\Phi$ is defined as in \eqref{eq:Set_Phi},
  is not a \ac{DITS}. 
\end{proposition}
\begin{proof}
 
\newnew{Let $\imap\colon \Ifshist \rightarrow \{l_1, l_2\}$ and define $\imap^{-1}(l_1) = \big\{\his_{k}=(\histu_{k-1}, \histy_k) \in \Ifshist \mid \histu_{k-1}=(u_i)_{i=1}^{k-1}, \, u_i=u \,\, \mathrm{for} \,\, 1 \leq i \leq k-1, \, \textrm{and} \, k > 3 \big\}$, and $\imap^{-1}(l_2)=\Ifshist \setminus \imap^{-1}(l_1)$. Then $\imap^{-1}(l_1)$ is the set of histories of length $k>3$ which correspond to applying the same action $u$ for $k-1$ times, and $\imap^{-1}(l_2)$ is its complement.}
Then, there
  exist sequences $\his_{k-2}=(\histu_{k-3},\histy_{k-2})$ and
  $\his_{k-1}=(\histu_{k-2},\histy_{k-1})$ such that
  $\his_{k-2}=\his_{k-1}\cat (u,y)$ and
  $\his_{k}=\his_{k-1}\cat (u,y)$ for which
  $\imap(\his_{k-2})=\imap(\his_{k-1})=l_2$ and
  $\imap(\his_k)=l_1$. Thus,
  \[
    \big\{([\his_{k-2}]_\imap, (u,y), [\his_{k-1}]_\imap), ([\his_{k-1}]_\imap, (u,y), [\his_{k}]_\imap) \big\} \in \Phi/\imap.
  \]
  Since $[\his_{k-2}]_\imap=[\his_{k-1}]_\imap$ and
  $[\his_{k-1}]_\imap \neq [\his_{k}]_\imap$, the transition
  corresponding to $([\his_{k-1}]_\imap, (u,y))$ is not unique; thus,
  $(\Ifshist/\imap, U \times Y, \Phi/\imap)$ is not deterministic.\qed
\end{proof}

Note that Proposition~\ref{prop:non_sufficient_label} holds also in
the case of a generic \ac{ITS} $(\Ifs, U \times Y, \phi)$, with
non-history \ac{Istate}s, if there exist $s,s',q,q' \in \Ifs$ such that
$\{(s,(u,y),s'), (q,(u,y),q')\} \in \Phi$, in which $\Phi$ is defined
using $\phi$ as in \eqref{eq:Set_Phi}. Then, any \ac{Imap} $\imap$
such that $\imap(s) = \imap(q)$ and $\imap(s') \neq \imap(q')$ results
in a quotient system that is not a \ac{DITS}.

\begin{remark}\label{rem:SuffDITS}
Whether the quotient system derived from $(\Ifshist, U \times Y, \phi)$
  is a \ac{DITS} depends on the sufficiency of $\imap$. In
  \citep[Proposition 4.5]{WeiSakLav22} it is shown that the quotient of
  a transition system $(S,\Lambda,T)$ by a labeling function $\sigma$ is
  a deterministic transition system
  if and only if $(S,\Lambda,T)$ is full%
  \endnote{A transition system
    $(S,\Lambda,T)$ is full, if $\forall s\in S, \lambda \in \Lambda$
    there exists at least one $s'\in S$ with $(s,\lambda,s')\in T$.}
  and $\sigma$ is sufficient.  
\end{remark}

As $\phi_{hist}$ is a function with
domain $\Ifshist \times (U\times Y)$, it is full, so the following
is implied by \citep{WeiSakLav22} as a special case:


\begin{proposition}[\textbf{A Quotient system is a DITS when the labeling is sufficient}] \label{prop:suff_label}
  Let $(\Ifshist/\imap, U \times Y, \Phi_{hist}/\imap)$ be the
  quotient of $(\Ifshist, U \times Y, \phi_{hist})$ by $\imap$, in
  which $\Phi$ is defined as in~\eqref{eq:Set_Phi}.
  Then $(\Ifshist/\imap, U \times Y, \Phi/\imap)$ is a \ac{DITS} if
  and only if $\imap$ is sufficient.
\end{proposition}

\begin{remark} \label{remark:isomorphism_quotient}
 For an \ac{Imap} $\imap\colon \Ifshist \to \Ifsder$,
the quotient $(\Ifshist/\imap, U \times Y, \Phi_{hist}/\imap)$ is isomorphic to
$(\Ifsder,U \times Y,\Phi_{der})$, in which
\[\Phi_{der}=\{(\imap(\his),(u,y),\imap(\his')) \mid (\his, (u,y), \his') \in \Phi_{hist}\}\]
\cite[\newnew{Proposition}~2.37]{WeiSakLav22}. Thus, we can use the labels
introduced by an $\imap$ as the new (derived) \ac{Ispace} and the
corresponding quotient system as the derived~\ac{ITS}.   
\end{remark}

Suppose \BS{an \ac{Imap}} $\imap$ is sufficient. Then, the derived \ac{ITS} is a \ac{DITS}, meaning that given an \ac{Istate} $\is_{k-1}$ in the derived space 
$\Ifsder$, and ($u_{k-1}$, $y_k$), $\is_{k}\in \Ifsder$
can be uniquely determined. Consequently, we can write the derived \ac{ITS} as
$(\Ifsder,U \times Y,\phi_{der})$ in which
$$\phi_{der}\colon \Ifsder \times (U \times Y) \to \Ifsder$$ is the
new information transition function. Therefore, we no longer need to
rely on the full histories and the history \ac{ITS} and can rely
solely on the derived \ac{ITS}. This is shown in the first two rows of
the following diagram:
\begin{equation}\label{eqn:dia}
  \begin{tikzcd}[row sep=0.3cm]
    \Ifshist \arrow{r}{u_1,y_2} \arrow{d}{\imap} &
    \Ifshist \arrow{r}{u_2,y_3} \arrow{d}{\imap} &
    \Ifshist \arrow{r}{u_3,y_4} \arrow{d}{\imap} &
    \Ifshist \arrow{d}{\imap} \arrow[dash, dotted]{r} & \phantom{.}\\
    \Ifsder \arrow{r}{u_1,y_2} \arrow{d}{\imap'} &
    \Ifsder \arrow{r}{u_2,y_3} \arrow{d}{\imap'} &
    \Ifsder \arrow{r}{u_3,y_4} \arrow{d}{\imap'} &
    \Ifsder \arrow{d}{\imap'} \arrow[dash, dotted]{r} & \phantom{.}\\
    \Ifsmin \arrow{r}{u_1,y_2} \arrow{d}{\imap''} &
    \Ifsmin \arrow{r}{u_2,y_3} \arrow{d}{\imap''} &
    \Ifsmin \arrow{r}{u_3,y_4} \arrow{d}{\imap''} & 
    \Ifsmin \arrow{d}{\imap''} \arrow[dash, dotted]{r} & \phantom{.}\\
    \Ifst  &
    \Ifst  &
    \Ifst  & \Ifst 
  \end{tikzcd}
\end{equation}
Note that we can similarly define an \ac{Imap} that maps any derived
\ac{Ispace} to another. An example is given in \eqref{eqn:dia} as the
mappings $\imap'\colon\Ifsder \to \Ifsmin$ and
$\imap''\colon\Ifsmin \to \Ifst$. 
\BS{In this example, $\imap'$ is sufficient, visible also from the commutativity of the respected square in the diagram. This implies that the quotient system derived by $\imap'$ is deterministic. On the other hand, $\imap''$ is not sufficient, meaning that the derived ITS is not deterministic: given an element of $\Ifstask$ one cannot uniquely determine the next \ac{Istate} using the derived ITS only. This is shown in \eqref{eqn:dia} with the missing arrows at the respective row of $\Ifstask$. Hence, for $\imap''$ the diagram does not commute.}
Note that an \ac{Imap} whose domain is $\Ifshist$ can also be defined as composition of the mappings along the column of the diagram. For example, $\imap_{min}\colon\Ifshist \to \Ifsmin$ is the
composition of $\imap$ and $\imap'$, that is, $\imap_{min}=\imap' \circ \imap$ (same for $\imapb\colon \Ifshist \to \Ifst$).

\subsection{Model-based and model-free} \label{Section_Model_based_model_free}

In machine learning, control, and robotics literature,
methods are often categorized into \emph{model-based} and
\emph{model-free} (or data-driven) ones. 
\new{Informally, using our setup, a model-based scenario} 
\newnew{is one where} \new{the derived \ac{Istate} is allowed to depend on knowledge about the external system, the sensor mapping, the initial state, and the disturbances acting on the external system or on the sensors (if there are any).}
The model-free scenario in contrast cannot depend on those, but can depend on data, which in our case is the history I-states, \newnew{that is}, 
the sequences of actions and observations.

\new{
In Section~\ref{ssec:IntExt} we have defined internal-external coupled systems. Their coupling \eqref{eqn:coupled_sys} produces an autonomous system \newnew{$(\Ifs\times X,\phi*_{\pi, h}f)$}. However, we can choose not to consider either one of the coupling functions $\pi$ and $h$, and be left with a system that still has a control parameter. For example, let \newnew{$(\Ifs\times X,U,\phi*_{h}f)$} be a system where the evolution of states can be written as 
\begin{align}
  \centering
  \is' &=\phi(\is,y) &\text{in which } y & = h(x), \nonumber \\
   x'  &=f(x, u). \label{eqn:coupled_sys_U}
\end{align}
Here, $u\in U$ is a control parameter on which the next state always depends.
This system represents the coupled system before a policy $\pi$ has been defined over the states of the internal.}

\newnew{Note that the internal system only has access to the current information state $\iota \in \Ifs$, not to the external state $x\in X$. One can notationally express this perspective by evolving the internal system by an externally parametrized information transition function  $\phi_{f, h}(\cdot \,;\, x)$, which maps the current I-state and action pair $(\iota, u) \in \mathcal{I} \times U$ to the next I-state $\iota' \in \mathcal{I}$. The maps $h$ (which couples the external to the internal system) and $f$ are subsumed into the global map $\phi_{f, h}$ which is additionally parametrized by the current state $x \in X$ of the external system. Thus, in accordance with~\eqref{eqn:coupled_sys_U}, we define $\phi_{f,h}$ for each $(\iota, u) \in \mathcal{I} \times U$ and $x \in X$ by
\begin{equation} \label{eqn:coupled_sys_U_cond}
\phi_{f, h}\big(\iota, u \, ; \, x\big) := \phi\big(\iota, h(f(x,u)) \big).
\end{equation}
If the I-space in~\eqref{eqn:coupled_sys_U} is the history I-space, we can write~\eqref{eqn:coupled_sys_U} as \newnew{$(\Ifshist\times X,U,\phi*_h f)$, and its} 
\new{internal system perspective}
\eqref{eqn:coupled_sys_U_cond} becomes \newnew{$(\Ifshist,U,\phi_{f,h})$}.
We propose that a method of obtaining a derived I-space corresponding to an I-map $\kappa$ is \emph{model-based}, if $\kappa$ is obtained as a function of \newnew{$(\Ifshist\times X,U,\phi*_h f)$}, while it is \emph{model-free}, if it is obtained as a function of \new{it from the perspective of the internal system, that is, as a function of}
\newnew{$(\Ifshist,U,\phi_{f,h})$}.
}

\newnew{The distinction between model-based and model-free is also} seen in the initial states $\eta_0$ of the history I-space. In model-based setups, typically $\eta_0$ is a subset of $X$, or a probability distribution over $X$ while in model-free setups $\eta_0$ is an empty sequence. Examples~\ref{ex:consistent_rot} and \ref{ex:LshapedCorridor} are examples of model-free and model-based \ac{Ispace}s respectively. 

Note that this formalization implies that model-free methods are a subset of model-based. This is because the internal perspective is itself a function of the entire coupled system, so anything that is a function of the internal perspective is by transitivity also a function of the entire coupled system. This matches the intuition that model-based are ones where more information is available. We leave the exploration of more aspects of this distinction and its formalization for future work.

We now present two examples that illustrate model-based and model-free derived ITSs. 

\begin{example}[\textbf{\emph{Bayesian filter}}] \label{eg:Bayesian_filter}
Suppose the initial history information state encodes a probability distribution over $X$ such that $\his_0=P(x_1)$. We refer to the coupled system including the disturbances described in \eqref{eqn:coupled_sys_dist}. A Markovian, probabilistic model of the disturbances is given \newnew{in the form of conditional distributions $P(\psi \mid x)$ over $\Psi$, and $P(\theta \mid x, u)$ over $\Theta$. In the former, conditioning takes place relative to external states $x \in X$, and in the latter relative to state-action pairs $(x,u)\in X \times U$}.
Using the definitions of $f$ and $h$ given in \eqref{eqn:coupled_sys_dist}, $P(y_k \mid x_k)$ and $P(x_{k+1} \mid x_k, u_k)$ can be derived from $P(\psi_k \mid x_k)$ and $P(x_{k+1} \mid x_k, u_k)$ for all stages $k$.

Let $\Ifs_{prob}$ be the set of all probability distributions defined over $X$ and let $\Ifshist$ be a history information space with $\Ifs_0=\Ifs_{prob}$ such that $\his_0$ is a probability distribution over $X$, that is $P(x_1)$. An \ac{ITS} can be derived by $\imap_{prob}\colon \Ifshist \rightarrow \Ifs_{prob}$ such that $\imap_{prob}(\his_k)=\is_k=P(x_k \mid \his_k)$. Note that we can write $\his_k$ as $\his_k=\his_{k-1}\cat (u_{k-1},y_k)$. The I-state $\is_k=P(x_k \mid \his_k)$ can be inductively computed from $\is_{k-1}$ and $(u_{k-1},y_k)$ using marginalization and Bayes' rule starting from $\is_1 = P(x_1\mid y_1)$, in which $\his_1=y_1$. This corresponds to defining $\phi \colon\ifs \times (U \times Y) \rightarrow \ifs$ such that $\is_{k}=\phi(\is_{k-1}, (u_{k-1},y_k))$. Then, \hspace{0.2cm} $\imap(\his_{k-1}\cat (u_{k-1}, y_k))=\phi\circ\imap(\his_{k-1})=P(x_{k}\mid \his_k)$ which shows that $\imap_{prob}$ is sufficient. Hence, a Bayesian filter can be modeled as a derived \ac{DITS} whose state space is $\ifs_{prob}$. Note that in this case, $\kappa_{prob}$ is defined as a function of \newnew{$(\Ifshist\times X,U,\phi *_h f)$}, making it model-based.
\end{example}

Note that the Kalman filter is a special case of a Bayesian filter when $f$ and $h$ are linear and the disturbances are Gaussian. These specifications \newnew{imply that all the posterior distributions are Gaussian as well}.
Therefore, in this special case, the range of $\imap_{prob}$ is implicitly restricted to the set of all Gaussian distributions, denoted as $\Ifs_{Gauss}$, such that $\imap_{prob}\colon \Ifshist \rightarrow \Ifs_{Gauss} \subset \Ifs_{prob}$. This restriction allows the \ac{Istate} to simply encode only the mean and the covariance of a multivariate Gaussian distribution, that is, $\is = (\hat{x},\Sigma)$, in which $\hat{x}$ is the mean and $\Sigma$ is the covariance matrix, without violating the sufficiency of $\imap_{prob}$. An extension of the Kalman filter to nonlinear systems is the Extended Kalman Filter (EKF). In the case of EKF, the functions $f$ and $h$ are not linear. This violates the posterior distribution being Gaussian even if the disturbances are. However, the states of the EKF are defined as elements of $\ifs_{Gauss}$ and a state transition function $\phi \colon \ifs_{Gauss}\times (U \times Y) \rightarrow \ifs_{Gauss}$ is described that relies on linearizing $f$ and $h$ at each \ac{Istate}. Note that even though the Kalman filter and the EKF share the same underlying \ac{Ispace}, the corresponding I-maps that derive these transition systems are different.

The following is an example of a model-free derived ITS.  

\BS{
\begin{example}[\textbf{Moving average filter}] Let $Y= \Re$ and $\imap_{k}: \ifs_k \rightarrow \Re$, in which $\ifs_k$ is the set of $k$ stage histories. A moving average filter (observation only) with a window size $n$ can be derived from $\ifs_k$ as
$$(\histu_{k-1},\histy_k) \mapsto \frac{1}{n}\sum_{i = k-n+1}^k y_i.$$
\end{example}}

\subsection{Lattice of information transition systems}

We fix $\Ifshist$, which corresponds to fixing the set of initial
states $\Ifs_0$. Then, each \ac{Imap} $\imap$ defined over $\Ifshist$
induces a partition of $\Ifshist$ through its preimages, denoted as
$\Ifshist / \imap$.

\begin{definition}[\textbf{Refinement of an \ac{Imap}}]
  An \ac{Imap} $\imap'$ is a \emph{refinement} of $\imap$, denoted as
  $\imap' \succeq \imap$, if $\forall A \in \Ifshist / \imap'$ there
  exists a $B \in \Ifshist / \imap$ such that $A \subseteq B$.
\end{definition}
Let $K(\Ifshist)$ denote the set of all partitions over
$\Ifshist$. Refinement induces a partial ordering since not all
partitions of $\Ifshist$ are comparable.  The partial ordering given
by refinements form a lattice of partitions over $\Ifshist$, denoted
as $(K(\Ifshist),\succeq)$.

At the top of the lattice, there is the partition induced by an
identity \ac{Imap} (or equivalently, by a bijection),
$\imap_{id}\colon  \Ifshist \rightarrow \Ifshist$, since all of its elements
are singletons (all equivalence classes contain exactly one element),
making it the maximally distinguishable case. Conversely, we can
define a constant mapping
$\imap_{const}\colon \Ifshist \rightarrow \Ifs_{const}$ for which
$\Ifshist/\imap_{const}$ is a singleton, that is,
$\Ifs_{const}=\{\is_{const}\}$, which then will be at the bottom of
the lattice. In turn, $\imap_{const}$ yields the minimally
distinguishable case as all histories now belong to a single
equivalence class.  This idea is similar to the notion of the
\emph{sensor lattice} defined over the partitions of $X$ see
\citep{Lav12b,ZhaShe21}. Indeed, if we take $\Ifs_0=X$ and consider
$\imap_{est}\colon \Ifshist \rightarrow X$, the ordering of partitions of
$\Ifshist$ such that $\Ifshist/\imap_{est}$ is the least upper bound
gives out the sensor lattice.

As motivated in previous sections, we are interested in finding a
sufficient \ac{Imap} such that the quotient \ac{ITS} derived from the
history \ac{ITS} is still deterministic. Notice that the constant
\ac{Imap} $\imap_{const}$ is sufficient by definition since for all
$(u,y) \in U\times Y$, and all $\his,\his' \in \Ifshist$, we have that
$\imap_{const}(\his)=\imap_{const}(\his')$ and
$\imap_{const}(\phi_{hist}(\his,(u,y)) =
\imap_{const}(\phi_{hist}(\his',(u,y)).$ On the other hand, in certain
cases it is crucial to differentiate certain histories from
others. This will become clear in the next section when we describe
the notion of a task.  Suppose $\imap$ is a labeling that partitions
$\Ifshist$ into equivalence classes that are of importance and suppose
that $\imap$ is not sufficient. Then, we want to find a refinement of
$\imap$ that is sufficient.  This will serve as a lower bound on the
lattice of partitions over $\Ifshist$ since for any partition such
that $\Ifshist/\imap$ is a refinement of it, the
classes of histories that are deemed crucial will not be
distinguished. The following defines the refinement of $\imap$ that
ensures sufficiency and a minimal number of equivalence classes.

\begin{definition}[\textbf{Minimal sufficient refinement}]
  Let $(\Ifshist,U \times Y, \fmaphist)$ be a history \ac{ITS} and
  $\imap$ an \ac{Imap}. A \emph{minimal sufficient refinement} of
  $\imap$ is a sufficient \ac{Imap} $\imap'$ such that there does not
  exist a sufficient \ac{Imap} $\imap''$ that satisfies
  $\imap' \succ \imap'' \succeq \imap$.
\end{definition}

\begin{remark}\label{rem:UniqueMSR}
  It is shown in \citep[Theorem 4.19]{WeiSakLav22} that the minimal sufficient refinement of $\imap$ defined over the states of a deterministic transition system $(S,\Lambda,\tau)$ is unique up to relabeling, namely if $\imap_{\min}$ and $\imap_{\min}'$ are minimal sufficient refinements, then $\imap_{\min} \succ \imap'_{\min}$ and $\imap'_{\min} \succ \imap_{\min}$.
\end{remark}

\section{Solving Tasks Minimally}\label{sec:problems}

\subsection{Definition of a task}

\new{In this section, we formulate general planning and filtering tasks within the framework of information transition systems.} We 
distinguish between two categories: 1) {\em active}, which \new{entails}
planning and executing an information-feedback policy that forces a desirable outcome in the \new{external system}, 
and 2) {\em passive}, which {refers to}
only observing the \new{external system}
without being able to effect changes.  
\newnew{We next describe active and passive tasks for the model-free and model-based I-space formulations, introduced in Section~\ref{Section_Model_based_model_free}.}
In the model-free case, tasks are specified using a logical language over
$\Ifshist$. This results
in a labeling, \new{a} derived I-space $\Ifst$,
and {the} associated I-map $\imapb$.
Various logics are allowable, such
as propositional, modal, or a temporal logic.  
The resulting sentences of the
language involve combinations of predicates that 
assign true or false values to subsets of $\Ifshist$.
Solving an active task 
requires that a sentence of interest becomes true during execution of the policy. This is called {\em satisfiability}. For example, the task may be to simply reach some goal set $G \subseteq \Ifshist$, causing a predicate {\sc in-goal}$(\Ifshist)$ to
become satisfied (in other words, be true).

Solving a passive task only requires maintaining whether a sentence is satisfied, rather than forcing an outcome; this corresponds to filtering. Whether the task is active or passive, if satisfiability is concerned with a single, fixed sentence, then a {\em task-induced labeling} (or {\em task labeling} for short), that is, $\imapb$, over $\Ifshist$ assigns two labels: Those I-states that result in true and
those that result in false.  
\new{A task labeling may also be assigned for a set of
sentences. In this case, each sentence induces a partition of $\ifshist$, and
the task labeling over $\Ifshist$ assigns a label to each set in the common refinement of these partitions.
}   
In the model-based case, tasks are instead
specified using a language over $X$, and sentence satisfiability must be determined by an I-map that converts history I-states into expressions over $X$.


\new{Some naturally occurring robot tasks can only be described in terms of infinite sequences of actions and observations.
These are called \emph{infinitary tasks}.
For example, cycling through a finite sequence of subsets of $X$ indefinitely while avoiding others \citep{FaiGirKrePap09} 
can only be described in terms of infinite histories. 
For this task, whether the sentence of interest is satisfied cannot be determined based on a finite history of any given length. However, the histories that fail, that is, those for which the sentence of interest becomes false, can be defined in terms of finite histories (namely those that result in a state that needed to be avoided). Interested reader can refer to \citep{KreFaiPap09} for examples based on linear temporal logic (LTL).
%

Infinitary tasks
are
defined on the set of infinite histories $\Ifs_{hist}^\infty$
which consists of infinite sequences of the form $\bar\eta = (\eta_0,y_1,u_1,y_2,u_2,\cdots)$. These are the elements of the infinite Cartesian product
\begin{equation*}
\ifs_0 \times (Y \times U) \times (Y \times U) \times \cdots \,
= \, \ifs_0 \times \prod_{k=1}^\infty (Y \times U).
\end{equation*}

The preimages of an \emph{infinitary task labeling} $\bar\kappa : \Ifs_{hist}^\infty \to \Ifstask$ are subsets of $\Ifs_{hist}^\infty$. Although the satisfiability of an infinitary task may depend on infinite sequences, these can nevertheless be characterised in terms of finite initial segments as follows. Any subset $H \subseteq \Ifs_{hist}^\infty$ can be written as
\begin{equation} \label{Eq_product_space_subset}
H = I_0 \times \prod_{k=1}^\infty (Y_k \times U_k),
\end{equation}
in which $I_0 \subseteq \Ifs_0$, and $Y_k \subseteq Y$, $U_k \subseteq U$ for all $k \in \N$. For each $m \in \N$,
we denote by $H(m)$ the collection of subsets of $\Ifs_{hist}^\infty$ for which $Y_k = Y$ and $U_k = U$ for all $k > m$, that is, 
\begin{multline*}
H(m) = \Big\{ I_0 \times \prod_{k=1}^m \left(Y_k \times U_k \right) \times \prod_{k=m+1}^\infty \left(Y \times U \right) \\
 \mid\, I_0 \subseteq \Ifs_0,\, Y_k \subseteq Y,\, U_k \subseteq U,\, k=1,\dots,m \Big\}.
\end{multline*}
In other words, such collections of histories are constrained only at a finite number of stages.

Now, let $\iota \in \Ifstask$ and suppose an equivalence class induced by the preimage $\bar\kappa^{-1}(\iota)$ is a (potentially infinite) union
\begin{equation} \label{Eq_union_of_base_sets}
\bar\kappa^{-1}(\iota) = \bigcup_{\alpha \in A} H_\alpha,
\end{equation}
in which $A$ is some index set and each $H_\alpha$ belongs to $H(m)$ for some $m$. Then, whether a particular history $\bar\eta \in \Ifs_{hist}^\infty$ belongs to $\bar\kappa^{-1}(\iota)$ is determined by a finite number of stages in an initial segment of this history. In general, however, the length of these initial segments is not bounded from above.

To characterize infinitary tasks in terms of deciding their satisfiability, we rely on topology. 
Assume that some topology is defined for the sets $\Ifs_0$, $Y$, and $U$. If these are finite sets, a natural choice is the discrete topology in which every singleton is an open set. For subsets of $\Re^n$, a natural choice would be the relative topology induced by the usual Euclidean topology in $\Re^n$. The base $H^\circ$ of the product topology in $\Ifs_{hist}^\infty$ consists of those sets $H$ for which $H \in H(m)$ for some $m \in \N$, and the sets $I_0$, $Y_k$, $U_k$ in~\eqref{Eq_product_space_subset} satisfy that $I_0$ is open in $\Ifs_0$, and $Y_k$, $U_k$ are open in $Y$, $U$, respectively, for all $k \in \N$.
All other open sets are obtained as arbitrary unions of sets $H \in H^\circ$. In particular, when the sets $H_\alpha$ in \eqref{Eq_union_of_base_sets} satisfy $H_\alpha \in H^\circ$ for all $\alpha \in A$, the corresponding preimage $\bar\kappa^{-1}(\iota)$ is an open set.
%
%
If the sets $\Ifs_0$, $Y$, $U$ are compact, the space $\Ifshist^\infty$ is compact in the product topology. This is the case for example when $\Ifs$, $Y$, and $U$ are finite sets with the discrete topology. %

\BS{In the simplest nontrivial case, a task labeling $\bar\kappa$ concerns a single sentence. Then, $\Ifstask = \{0,1\}$ so that the preimages of $\bar\kappa$ partition $\Ifshist^\infty$ into the equivalence classes $\bar\kappa^{-1}(1)$, that is, the set of histories for which the sentence is satisfied, and $\bar\kappa^{-1}(0)$, the set of histories for which the sentence is false.
We call $\bar\kappa^{-1}(1)$ and $\bar\kappa^{-1}(0)$ the \emph{success set} and \emph{fail set}, respectively.}
If the success set of a given task is open, we call this an \emph{open task}.
A \emph{closed task} is one whose fail set is open, so that its success set is closed. 
It is possible for a task to be both open and closed, so that both the success and fail sets are both open and closed. We call such tasks \emph{clopen}.

\BS{Due to the definition of the sets $H(m)$, the membership of a given history in an open set}
\KT{is determined by some} \BS{finite initial segment of that history. Therefore, based on a finite length segment of a given history, we can determine its membership in the success set of an open task, in the fail set of a closed task, and both the success and fail sets for a clopen task. In these cases, task satisfiability can be defined} 
\KT{in terms of}
\BS{the elements of $\ifshist$. We can thus transcribe an infinitary task labeling $\bar\kappa: \ifshist^\infty \rightarrow \Ifstask$ in the form of $\kappa: \ifshist \rightarrow \Ifstask$.} 
\KT{This amounts to assigning to each finite history $\eta \in \Ifshist$ some task label $\iota \in \Ifstask$ in such a way that this labeling expresses the success set of the corresponding infinitary task labeling $\bar\kappa$. 
Recall that $\Ifshist = \bigcup_{k \in \N} \mathcal{I}_k$, in which $\mathcal{I}_k$ are as in~\eqref{Eq_Def_finite_Ifshist}.
Suppose $\bar\kappa^{-1}(\is)$ is an open set for some $\iota \in \Ifstask$ so that $\bar\kappa^{-1}(\is) = \bigcup_{\alpha \in A}H_\alpha$ as in~\eqref{Eq_union_of_base_sets}, with $H_\alpha$ open for all $\alpha \in A$. We may assume that $H_\alpha \in H^\circ$ for all $\alpha \in A$. Then, for a finite history $\his = (\eta_0, y_1, u_1, \ldots, u_{k-1}, y_k) \in \mathcal{I}_k$ 
we define $\kappa(\his) = \is$ if and only if there exists some $m \leq k$ and some index $\alpha \in A$ for which the corresponding set 
\[
\hspace{-0.4mm} H_\alpha \hspace{-0.25mm} = \Ifs_0 \hspace{-0.25mm} \times \hspace{-0.25mm} (Y_1 \hspace{-0.25mm} \times \hspace{-0.25mm} U_1) \hspace{-0.25mm} \times \hspace{-0.25mm} \cdots \hspace{-0.25mm} \times \hspace{-0.25mm} (Y_m \hspace{-0.25mm} \times \hspace{-0.25mm} U_m) \hspace{-0.25mm} \times \hspace{-2.5mm}\prod_{n=m+1}^\infty \hspace{-1mm} (Y \times U),
\]
satisfies 
$y_n \in Y_n$ and $u_n \in U_n$ for all $1 \leq n \leq m$.}

Below are examples of typical model-based task descriptions with their corresponding definitions in terms of infinite histories \BS{that can be expressed using a task-labeling over $\ifshist$}. 

\begin{example}[\textbf{\emph{Reach state $x \in X$ from an initial state $x_0 \in X$}}]\label{ex:open_task}
The success set of this task consists of those histories that correspond to the external system arriving to $x$ in some finite time. If $H_m \subseteq \Ifshist^\infty$ contains those histories in which $x$ is visited for the first time at stage $m$, the success set of this task is $\bar\kappa^{-1}(1) = \bigcup_{m \in \N} H_m$. Assuming $H_m \in H^\circ$ for every $m \in \N$, this task is open.
\end{example}

\begin{example}[\textbf{\emph{Never visit state $x \in X$}}]\label{ex:closed_task}
The fail set of this task consists of those histories that correspond to the external system arriving at $x$ from some initial state $x_0$. Since this always happens in finite time (or not at all), the fail set can be expressed as the union $\bigcup_{m\in\N} H_m$, where $H_m$ consists of all the histories in which state $x$ is reached for the first time on stage $m$. Assuming $H_m \in H^\circ$ for all $m \in \N$, the task is therefore closed.  
\end{example}

\begin{example}[\textbf{\emph{Reach state $x_1$ while avoiding state $x_2$}}] \label{ex:clopen_task}
The success set of this task may be written as the union, over $m \in \N$, of histories which reach $x_1$ for the first time after $m$ stages, and did not visit $x_2$ during the first $m-1$ stages. Assuming these sets are open, the success set is thus open. With an analogous argument, it is seen that the fail set is also open.
\end{example}

\BS{
An infinitary task is not necessarily either open or closed. One example of this are tasks that can be expressed as so-called $G_\delta$ sets~\citep[Section 3.6]{Dug78}, that is, infinite intersections of open sets (see Example~\ref{ex:G_delta_task}).}

\begin{example}[\textbf{\emph{Revisit
state $x \in X$ infinitely many times}}]\label{ex:G_delta_task}
In this case, neither the success set or fail set of this task can be defined in terms of the sets $H(m)$, since no finite length history can rule out either success or failure. For each $m,k \in \N$, define $H_{m,k} = \{ \bar\eta \in \Ifshist^\infty \, \mid \, \textrm{at stage}\, m, \, \textrm{the next visit to} \, x\, \textrm{happens after}\, k\, \textrm{stages} \}$. Then the success set is given by $\bar\kappa^{-1}(1) = \bigcap_{m \in \N} \bigcup_{k \in \N} H_{m,k}$ which is an infinite intersection of open sets if $H_{m,k}\in H^\circ$ for all $m,k$.
This set is not generally open, but belongs to the broader class $G_\delta$.
\end{example}
}

\BS{In this paper, we consider tasks that are expressed as a labeling over $\ifshist$ or those over $\ifshist^\infty$ that can be transcribed as one over $\ifshist$. Hence, the problem families that we will introduce in Section~\ref{ssec:prob_family} would refer to these types of tasks.}
All the tasks in examples in Section~\ref{sec:ex} are either open or
closed and thus are representable by either a finitary success or
failure condition. More generally, if one defines tasks using any
common version of temporal logic, the corresponding success sets are
always going to be Borel, that is, members of the sigma algebra generated by open sets~\citep[Section 3.6]{Dug78}.

\subsection{Problem families}\label{ssec:prob_family}

It is assumed that the state-relabeled transition system
$(X, U, f, h,Y)$ describing the external system is fixed, but it is
unknown or partially known to the observer (a robot or other
observer).

Filtering (passive case) requires maintaining the label of an
\ac{Istate} attributed by $\imapb$. Since $\imapb$ is not necessarily
sufficient, we cannot guarantee that the quotient system by $\imapb$
is a \ac{DITS} (Propositions~\ref{prop:non_sufficient_label} and
\ref{prop:suff_label}). This implies that relying solely on the quotient system by
$\imapb$, we cannot determine the class that the current history
belongs to (see the last row in \eqref{eqn:dia}). Hence, we cannot
determine whether a sentence describing the task is satisfied (or
which sentences are satisfied).

Suppose the sets $U$ and $Y$ are specified, and at each stage $k$, the action $u_{k-1}$ is known and $y_k$ is observed. The following describes the
problem for a passive task given a state-relabeled (history) \ac{ITS}
$(\Ifshist, U \times Y, \phi_{hist}, \imapb, \Ifstask)$, in which
$\imapb\colon\Ifshist \to \Ifstask$ is a task labeling (that is not
assumed to be sufficient), and $\Ifstask$ is the corresponding~\ac{Ispace}.

\begin{problem}[\textbf{Find a sufficient \ac{Ispace} filter}]\label{prob:suff_ref}
  Find a sufficient refinement of $\imapb$.
\end{problem}

Note that $\Ifshist/\imapb$ determines a lower bound on the
partitioning of $\Ifshist$ which is interpreted as the crucial
information that cannot be relinquished without losing predictability, or
success guarantees. Consequently, histories belonging to
different equivalence classes with respect to $\imapb$ must always be
distinguished from each other.

\begin{example}[\textbf{\emph{Goal recognition}}]\label{ex:goal_task}
  Suppose $\imapb\colon \Ifshist \to \Ifstask$ is a labeling that
  partitions $\Ifshist$ into two disjoint sets;
  $\Ifstask=\{\is_G, \is_{NG}\}$, in which $\imapb^{-1}(\is_G)$ and
  $\imapb^{-1}(\is_{NG})$ correspond to histories that lead to goal
  and the ones that do not, respectively. Suppose the goal is
  recognizable, meaning that, solely based on $y_k$, the value of
  $\imapb(\his_k)$ is known, for all $\his_k \in \Ifshist$ and
  $k>0$. Then, $\imapb$ is trivially sufficient (also
  minimal). However, if the sensor mapping does not directly provide
  this information, then a refinement is needed to describe a
  sufficient filter that infers whether the goal is reached.
\end{example}        

Notice that 
Problem~\ref{prob:suff_ref} does not impose 
an upper bound. At the limit, a bijection from $\Ifshist$ is always a
sufficient refinement of $\imapb$. As stated previously, using history
\ac{ITS} can create computational obstructions in solving
problems.
This motivates the following problem.

\begin{problem}[\textbf{Find a minimal sufficient I-filter}]\label{prob:min_suff_ref}
  Find a minimal sufficient refinement of $\imapb$.
\end{problem}

We now consider a basic planning problem for which $\Ifst=\{0,1\}$,
such that $\imapb^{-1}(1)\subset \Ifshist$ is the set of histories
that achieve the goal, and $\imapb^{-1}(0)\subset \Ifshist$ is its
complement.  Most planning problems refer to finding a labeling
function $\pi$ such that, when used to label the states of the
internal system, guarantees task accomplishment. Then $\pi$ is called
a {\em feasible} policy, which is formally defined in the following.
Consider an external system $(X, U, f, h, Y)$.  Let
$\mathcal{R}_X(\Ifstask) \subseteq X$ be the set of initial states for
which there exist a $k \in \N$ and histories $\histx_k$, $\histu_{k-1}$, and
$\histy_k$, such that $x_{i+1}=f(x_i,u_i)$ and $y_i=h(x_i)$ for all
$0<i<k$, and $\his_k \in \imapb^{-1}(1)$, in which $\his_k$ is the
history \ac{Istate} corresponding to $\histu_{k-1}$ and $\histy_k$. \BS{Informally, $\mathcal{R}_X(\Ifstask)$ is the set of initial states of the external system for which there exists an action sequence such that the evolution of the external system under this action sequence results in histories that satisfy the task description. We will then call $\mathcal{R}_X(\Ifstask)$ the \emph{backward reachable set for $\Ifstask$}, analogously to the use of the same term in control theory.}

\begin{definition}[\textbf{Feasible policy for $\Ifst$}]
Let $(\Ifs, Y, \phi, \pi, U)$ and $(X,U,f, h, Y)$ be the state-relabeled transition systems corresponding to internal and external systems, respectively. A labeling function $\pi: \Ifs \to U$ 
is a feasible policy for $\Ifst$ if for all \BS{$x$ in the backward reachable set for $\Ifstask$, that is,} $x \in \mathcal{R}_X(\Ifst)$,
at least one history $\eta_k$ corresponding to the coupled 
internal-external system~\eqref{eqn:coupled_sys}
initialized at $(\is_0, x)$  belongs to~$\imapb^{-1}(1)$.
\end{definition}

Most problems in the planning literature consider a fixed \ac{DITS}
and look for a feasible policy for $\Ifst$. This yields the following
problem. Typically, the \ac{Ispace} considered is $X$ which makes the
resulting $\pi$ a \emph{state-feedback policy}\endnote{The term \emph{state} in state-feedback policy refers to the states of the external system, as it is commonly used in the robotics literature.}. Note that a \ac{DITS},
in other words, the robot brain, is an \ac{Ispace} filter itself.

\begin{problem}[\textbf{Find a feasible policy}]
  \label{prob:planning_wDITS}
  Given $(\Ifs, Y, \phi)$, find a labeling function
  $\pi\colon \Ifs \to U$ that is a feasible policy for~$\Ifst$.
\end{problem}

We can further extend the planning problem to consider an unspecified internal system. This entails finding a \ac{DITS} $(\Ifs, Y, \phi)$ and a policy $\pi\colon \Ifs \rightarrow U$ such that the resulting histories of the coupled system $(\Ifs\times X,\phi*_{\pi,h}f)$ belong to $\imapb^{-1}(1)$, that is, they satisfy the task description. This is the problem of jointly finding an \ac{Ispace}-filter and a feasible policy defined over its states. 
Let $\mathcal{K}$ be the set of all I-maps defined over $\Ifshist$. For $\imap \in \K$, let $\Pi_\imap$ be the set of all policies (labeling functions) that can be defined over the states of $\ifs$ which is the image of the mapping $\imap \colon \ifshist \rightarrow \ifs$.

\begin{problem}[\textbf{Find a DITS and a feasible policy}]
  \label{prob:planning_noDITS} 
  Find a pair 
  $(\imap, \pi)\in\{(\imap, \pi)\mid\imap\in\mathcal{K}\land\pi\in\Pi_\imap\}$ 
  such that $\imap$ is sufficient and $\pi$ is a feasible policy for~$\Ifstask$.
\end{problem}
Suppose $\imap\colon \ifshist \rightarrow \ifs$ and assume $(\kappa, \pi)$ is a solution to 
Problem~\ref{prob:planning_noDITS}. This corresponds to the DITS $(\ifs, Y, \phi_\pi)$ 
and a feasible policy $\pi\colon\ifs \rightarrow U$ such that $(\ifs, Y, \phi)$ is the 
derived ITS by $\imap$ and $(\ifs, Y, \phi_\pi)$ is the restriction of it by~$\pi$.

We emphasize that finding a \ac{DITS} for a planning problem differs
from Problems~\ref{prob:suff_ref} and \ref{prob:min_suff_ref} in the
sense that we are not looking for a refinement of $\imapb$. The reason
for this difference is because $\imapb$ can already be sufficient,
hence, it is the minimal sufficient refinement of itself. However,
this does not necessarily imply the existence of a feasible policy
defined over $\Ifstask$. For example, consider the $\imapb$ described
in Example~\ref{ex:goal_task} and a sensor mapping that reports whether
the goal is reached or not. Even though $\imapb$ is sufficient in this case,
knowing when the goal is reached does not imply, in most cases, that a
feasible policy exists as a labeling function for the quotient system
by $\imapb$, that is, over the states of $\Ifst$.
On the other hand, we can still talk about a notion of
minimality. This notion is defined in the following.

\begin{definition}[\textbf{Minimal DITS for $\mathbf{\pi}$}]\label{def:min_DITS_pi}
Let $\imap\colon \ifshist \rightarrow \ifs$ and $\pi\colon\ifs \rightarrow U$ be a solution to Problem~\ref{prob:planning_noDITS}. Furthermore, let $(\ifshist, Y, \phi_{hist,\pi\circ\imap})$ be the restriction of $(\ifshist, U\times Y, \phi_{hist})$ by $\pi \circ \kappa$. Denote by $(\ifshist', Y, \phi_{hist,\pi\circ\imap})$ the subgraph of $(\ifshist, Y, \phi_{hist,\pi\circ\imap})$ from which the nodes that are not reachable from $\his_0$ have been pruned. We restrict the domain of I-maps $\imap$ and $\imap'$ to $\ifshist'$.
Then, $(\ifs, Y, \phi, \pi, U)$, determined by $\imap$ and $\pi$, is \emph{minimal for $\pi$} if there does not exist a sufficient \ac{Imap} $\imap'$ with $\imap \succ \imap'$ and a corresponding policy $\pi'$ for the quotient system by $\imap'$ that satisfy $\pi\circ \imap = \pi' \circ \imap'$.

\end{definition}

Informally, a minimal \ac{DITS} for $\pi$ implies that one cannot further reduce the quotient system by merging equivalence classes induced by $\imap$, while simultaneously ensuring that when coupled to the external system that is initialized at the same state, the coupling would result in the same observation and action histories as $(\Ifs, Y, \phi, \pi, U)$.

There may be multiple pairs of $(\imap, \pi)$ that solve the same
problem.  Given two DITS,
$(\Ifshist/\imap_1, Y, \Fmaphist/\imap_1, \pi_1, U)$ and
$(\Ifshist/\imap_2, Y, \Fmaphist/\imap_2, \pi_2, U)$, a notion of
equivalence can be determined if the preimages of
$\pi_1 \circ \imap_1$ and $\pi_2 \circ \imap_2$ partition $\Ifshist$
in the same way. We can say that
$(\Ifshist/\imap_1, Y, \fmaphist, \pi_1, U)$ requires more histories
to be distinguished if the partitioning induced by
$\pi_1 \circ \imap_1$ is a refinement over the partitioning induced
by $\pi_2 \circ \imap_2$.

Suppose a feasible policy $\pi\colon \ifshist \rightarrow U$ is defined over the states of the history ITS $(\ifshist, U\times Y, \phi_{hist})$. The restriction of the history ITS to the policy $\pi$ is then $(\ifshist, Y, \phi_{{hist}_\pi})$ (recall the definition given in Section~\ref{sec:ITS_prespectives} of restriction of a DITS).
This is a particular case that solves Problem~\ref{prob:planning_noDITS} for which $(\imap, \pi)$ is the pair such that $\imap\colon \ifshist \rightarrow \ifshist$ is a bijection. Let $(\ifshist', Y, \phi_{{hist}_\pi})$ be the restriction by $\pi$ from which the states that are not reachable from $\his_0$ are pruned. 
Note that $\ifshist' \newnew{\subseteq} \ifshist$ is the set of histories that can be realized once the history ITS is restricted by the policy~$\pi$. Restricting the domain of $\pi$ to $\ifshist'$, we obtain a labeling function over the states of $(\ifshist', Y, \phi_{{hist}_\pi})$ which determines the classes of histories that
\BS{are distinguished by the actions selected under the policy $\pi$. To ensure that the same action histories are obtained when a derived \ac{DITS} (quotient of the history ITS by $\imap'$) is coupled to the external system, the \ac{Imap} $\imap'$ needs to be a refinement of $\pi$. Consequently, the following proposition establishes the connection between a policy $\pi$ defined over $\ifshist$ and its respective minimal DITS.}

\begin{proposition}[\textbf{The minimal sufficient refinement of a feasible policy $\pi: \ifshist \rightarrow U$ determines its minimal DITS}]
\label{prop:min_DITS_pi}
Let $(\imap, \pi)$ be a pair that solves Problem~\ref{prob:planning_noDITS} such that $\imap\colon \ifshist \rightarrow \ifshist$ is a bijection. Then, a minimal DITS for $\pi$ is the DITS $(\ifs, Y, \phi)$ derived from $(\ifshist', Y, \phi_{{hist}, \pi})$ by some minimal sufficient refinement $\imap'$ of~$\pi$. 
\end{proposition}
\begin{proof}
Since $\imap'$ is a minimal sufficient refinement of $\pi$, it is sufficient and $\nexists \imap''$ that satisfies $\imap' \succeq \imap'' \succeq \pi$. Since it is a refinement, every set in $\ifshist'/\imap'$ is a subset of $\ifshist'/\pi$. Thus, we can find a $\pi'\colon\ifs \to U$ such that $\pi'(\imap'(\his))= \pi(\his)$. Then, by Definition~\ref{def:min_DITS_pi}, $(\ifs, Y, \phi)$ labelled with $\pi'$ is a minimal DITS for $\pi$. \qed
\end{proof}

\subsection{Learning a sufficient ITS}

Although learning and planning overlap significantly, some unique
issues arise in pure learning \citep{WeiSakLav22}.
This
corresponds to the case when $\imapb\colon\Ifshist \rightarrow \Ifst$ is
not initially given but needs to be revealed through interactions with
the external system, that is, respective action and observation
histories. It is assumed that whether the sentence (or sentences)
describing the task is satisfied or not can be assessed at a
particular history \ac{Istate}.

We can address both filtering and planning problems defined previously
within this context, considering model-free and model-based cases. In
the model-free case, the task is to compute a minimal sufficient ITS
that is consistent with the actions and observations. Variations
include {\em lifelong learning}, in which there is a single, long
history I-state, or more standard learning in which the system can be
restarted, resulting in multiple {\em trials}, each yielding a
different history I-state. In the model-based case, partial
specifications of $X$, $f$, and $h$ may be given, and unknown
parameters are estimated using the history I-state(s).  Different
results are generally obtained depending on what assumptions are
allowed.  For example, do identical history I-states imply identical
state trajectories? If not, then set-based, nondeterministic models
may be assumed, or even probabilistic models based on behavior
observed over many trials and assumptions on probability measure,
priors, statistical independence, and so on.

\section{Applying the Theory}\label{sec:ex}

In this section we provide some simple filtering and planning problems
and show how the ideas presented in the previous sections apply to
these problems.  All problems defined in the previous section can be
posed in a 
learning context as well. Then, $\Ifst$ is not given but it is
revealed through interactions between the internal and external as the
input-output data. 
Finally, we formulate as derived ITSs two established approaches, \emph{diversity-based inference} and \emph{predictive state representations}, for obtaining compact representations of the input-output (action-observation) relations for an unknown external system. These techniques illustrate the model-free approach to representing the internal-external coupling.

\subsection{Red-green gates}

\begin{figure}[t!]
    \centering
    \includegraphics[width=0.5\linewidth]{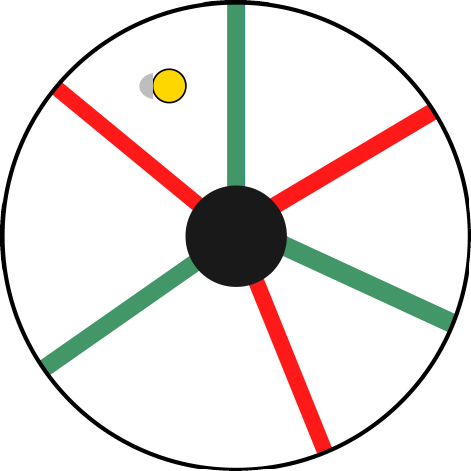}
    \caption{Environment used in
    Examples~\ref{ex:consistent_rot},\ref{ex:plan_consistent_rot}, the
    obstacle (an open disk) is shown in black.}
    \label{fig:red_green_env}
\end{figure}

This example is inspired by \citep{TovCohLav08}. Let $E\subseteq \Re^2$ be an 
annulus that is partitioned into non-empty regions
separated by gates, see Figure~\ref{fig:red_green_env}.  
Each gate is either green or red. This color can
be detected by the robot's color sensor and follows the rule that each region shares a boundary with exactly two gates; one green and one red.  The set of possible observations are therefore $Y=\{r,g\}$. 
As in~\citep{TovCohLav08} we assume that the robot's trajectory is in general position with respect to the gates,
in the sense that it only crosses them transversally and never goes through an intersection of two gates.

\begin{figure*}[t!]
  \centering
  \subfigure[]{\includegraphics[width=.48\linewidth]{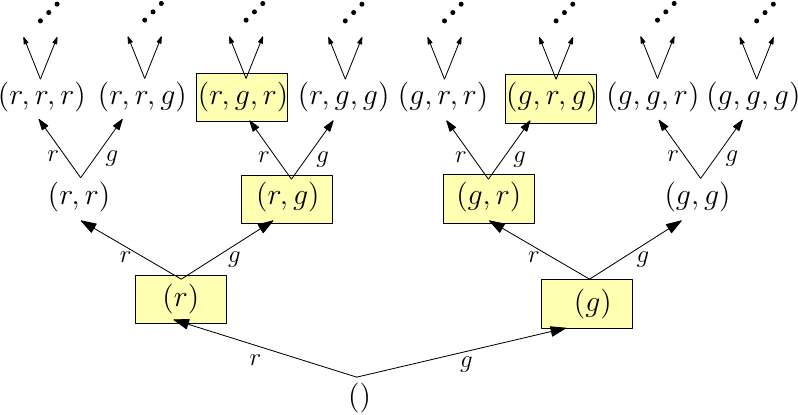}} 
  \subfigure[]{\includegraphics[width=.48\linewidth]{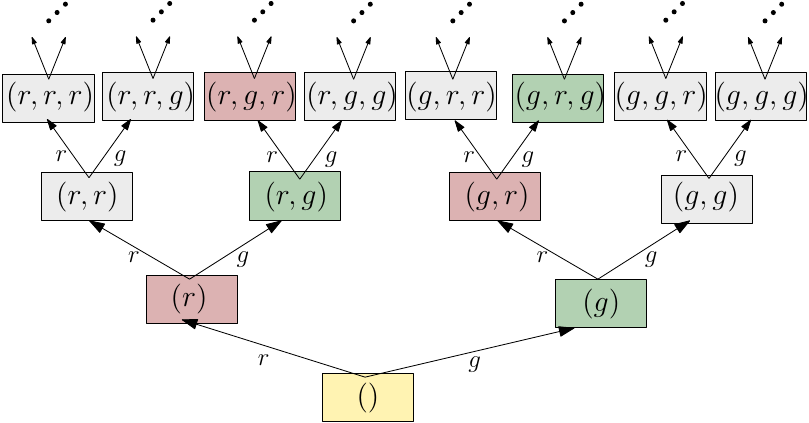}}  
  \caption{(a) State-relabeled history \ac{ITS} described in
    Example~\ref{ex:consistent_rot}, and the labeling function $\imapb$.
    States colored yellow 
    \BS{are the ones that do not violate the task description.} 
    (b) Equivalence classes induced by $\imap'$; the minimal sufficient refinement of $\imapb$. }
  \label{fig:red_green_gates}
\end{figure*}

\begin{example}[\textbf{\emph{Consistent rotation filter}}]
  \label{ex:consistent_rot} 
  This example considers a filtering problem from the perspective of
  an independent observer. Suppose the actions taken by the robot are
  not observable and the only information about the system is the
  history of readings coming from the robot's color sensor; for
  example, $(r,r,r,g,r,g)$.  Then, the history \ac{Ispace} is the set
  of all finite length sequences of elements of $Y$, that is,
  $\Ifshist=Y^*$, which refers to the free monoid generated by the
  elements of $Y$ (or the Kleene star of $Y$). Hence, the history
  \ac{ITS} can be represented as an infinite binary tree. The task is
  to determine whether the robot crosses the gates consistently (in a
  clockwise or counterclockwise manner) or not. The preimages of
  $\imapb\colon \Ifshist \to \Ifstask$ partition $\Ifshist$ into
  two subsets: one which the condition is satisfied (so far) and the
  others. The labeling induced by $\imapb$ is shown in
  Figure~\ref{fig:red_green_gates}(a).
\end{example}

\begin{claim}
  Task labeling $\imapb$ defined in Example~\ref{ex:consistent_rot} is
  not sufficient.
\end{claim}
\begin{proof}
  There exist \ac{Istate}s $\his, \his'$ such that
  $\imapb(\his)=\imapb(\his')$ and there exists a $y$ for which
  $\imapb(\fmaphist(\his,y)) \neq \imapb(\fmaphist(\his',y))$; for
  example consider $\his=(r,g)$, $\his'=(r,g,r)$ and $y=g$. This shows
  that $\imapb$ violates Definition~\ref{def:sufficiency}.\qed
\end{proof}

\begin{figure}
    \centering
    \includegraphics[width=.6\linewidth]{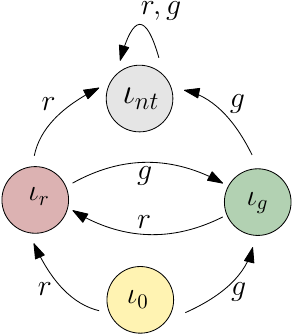}
    \caption{Quotient by $\imap'$ of the state-relabeled history \ac{ITS} shown in Figure~\ref{fig:red_green_gates}(b).}
    \label{fig:red_green_quotient_filter}
\end{figure}

\newnew{We can obtain a sufficient refinement of $\imapb$, defined as $\imap'\colon \Ifshist \to \{\is_0, \is_r, \is_g, \is_{nt}\}$. The corresponding equivalence classes are shown in Figure~\ref{fig:red_green_gates}(b). Its quotient \ac{DITS} is shown in Figure~\ref{fig:red_green_quotient_filter}.}

\begin{claim}
  $\imap'$ as defined above is a minimal sufficient refinement of~$\imapb$.
\end{claim}
\begin{proof}
  It follows from Proposition~\ref{prob:min_suff_ref} that if a
  labeling is not minimal then there is a minimal one that is strictly
  coarser and is still sufficient. However, neither of the subsets
  that belong to $\Ifshist/\imap'$ can be merged, since merging
  $\is_{nt}$ (colored gray in \new{Figure~\ref{fig:red_green_quotient_filter}}) with
  anything else violates the condition that $\imap'$ is a refinement
  of $\imapb$ and any pairwise merge of the others violate
  sufficiency. \qed
\end{proof}

Suppose the robot has a boundary detector, and it is capable of executing a bouncing motion that involves \emph{move forward} and \emph{rotate in place}. The set of actions is defined as $U=\{u_r, u_g\}$, in which $u_g$ represent a bouncing motion that allow the robot to traverse the green gate but not the red one, $u_r$ allows it to traverse the red gate but not the green one. For all the actions, the robot also bounces off of the boundary. We assume that the boundary detector and color sensor readings do not arrive simultaneously, and that the resulting trajectory will strike every open interval in the boundary of every region infinitely often, with non-zero, non-tangential velocities~\citep{BobSanCzaGosLav11}.

\begin{figure}
    \centering \hspace{-7.5mm}
    \includegraphics[width=1.06\linewidth]{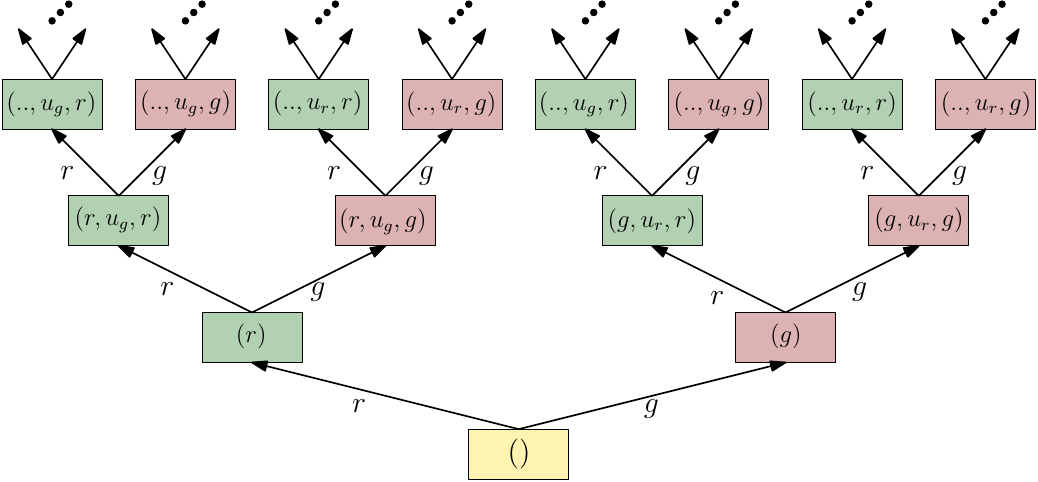}
    \caption{History ITS $(\Ifshist, U\times Y, \phi_{hist})$ restricted by $\pi$, labelled with $\pi$. The histories where $u_g$ is applied is colored green and where $u_r$ is applied in red. The initial state $\his_0$ is labeled with $()$.}
    \label{fig:red_green_policy_label}
\end{figure}

\begin{figure}
    \centering
    {\includegraphics[width=.5\linewidth]{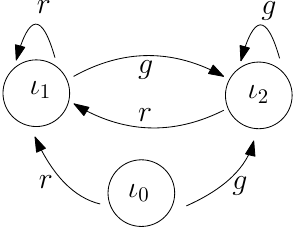}}
    \caption{\ac{DITS} describing the internal system solving the planning problem described in Example~\ref{ex:plan_consistent_rot}.}
    \label{fig:red_green_plan}
\end{figure}

\begin{example}[\textbf{\emph{Consistent rotation plan}}]\label{ex:plan_consistent_rot} 
We now consider a planning problem (that belongs to the class described in Problem~\ref{prob:planning_noDITS}) for which the goal is to ensure that the robot crosses the gates consistently. The history \ac{Ispace} of the planner is $\Ifshist=(U \times Y)^*$ and the preimages of $\imapb$ partition $\Ifshist$ into two sets; the histories that satisfy the predicate and the ones that do not. 
  
A policy $\pi \colon \ifshist \rightarrow \ifs$ can be determined over the states of history ITS such that $\pi(\his_0,\dots,y_k)=u_g$ if $y_k=r$ and $\pi(\his_0,\dots,y_k)=u_r$ if $y_k=g$. 
Let $(\ifshist', Y, \phi_{hist,\pi})$ be the restriction of history ITS by $\pi$ such that the states that are not reachable from $\his_0$ are pruned (see Figure~\ref{fig:red_green_policy_label}). The labeling $\pi$ defined over the states of $(\ifshist', Y, \phi_{hist,\pi})$ is sufficient as can be seen from the inspection of the ITS given in Figure~\ref{fig:red_green_policy_label}. \BS{Then, the following claim follows from Proposition~\ref{prop:min_DITS_pi}. 

\begin{claim}
The minimal DITS for $\pi$ is the quotient of $(\ifshist', Y, \phi_{hist,\pi})$ by $\pi$.  
\end{claim}}
The quotient system, that is the minimal DITS, is shown in Figure~\ref{fig:red_green_plan}. Let $\ifs = \{i_0,i_1,i_2\}$ be the states of this quotient ITS. The respective plan $\pi'$ represented over the states of this minimal DITS is given as $\pi'(\is_0)=()$, $ \pi'(\is_1)=u_g$, $\pi'(\is_2)=u_r$.
  
\end{example}

\subsection{L-shaped corridor}

Consider a robot in an inverted L-shaped planar corridor
(Figure~\ref{fig:L_corridor}). Let $\mathcal{E}_l$ be the set
of all such environments such that $l_1,l_2 \leq l$, in which $l_1$
and $l_2$ are the dimensions of the corridor bounded by $l$. We assume
that the minimum length/width is larger than the robot radius, that
is, $1$. The state space $X$ is defined as the set of all pairs
$(q, E_i)$, in which $(q_1,q_2) \in E_i$, and
$E_i \in \mathcal{E}_{l}$. The action set is \new{one}
which corresponds to moving one step towards right or up; if the
boundary is reached, the state does not change. The robot has a sensor
that reports $1$ if the motion is blocked.

\begin{figure}
    \centering
    \includegraphics[width=.6\linewidth]{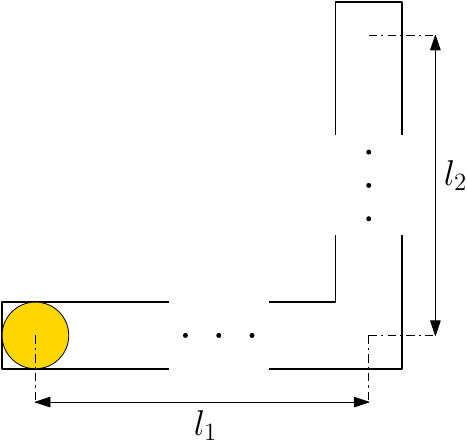}
    \caption{L-shaped corridor;
    $l_1,l_2\leq l$. For any corridor, the robot starts at the left-most part of the corridor which corresponds to the coordinates~$(0,0)$.}
    \label{fig:L_corridor}
\end{figure}

\begin{example}[\textbf{\emph{L-shaped corridor}}]
  \label{ex:LshapedCorridor}
  Consider a model-based history \ac{ITS} with $\his_0 \subset X$ that
  specifies the initial position as $q_0=(0,0)$ which corresponds to
  the left-most bottom square of the mirrored L-shape
  \new{(Figure~\ref{fig:L_corridor})} but does not specify the
  environment so that it can be any $E_i \in \E_l$.  Let $\Ifshist$ be
  its set of states
  and let $\imap_{ndet}\colon\Ifshist \to \pow(X)$ be an \ac{Imap} that
  maps the history \ac{Istate} $\his_k$ at stage $k$ to a subset of
  $X_k \subseteq X$. 
  Since $(X,U,f,h,Y)$, and $X_0=\his_0$ are known, transitions for the
  quotient system can be described by induction as
  $X_{k+1}=\hat{X}(X_k,u_k) \cap H(y_{k+1})$, in which
  $\hat{X}(X,u) := \{f(x,u)\mid x\in X\}$, and $H(y) := h^{-1}(y)\subseteq X$
  is the set of all states that could yield~$y$.
  By construction,
  $\imap_{ndet}$ is sufficient.
  Suppose $\imapb\colon \Ifshist/\imap_{ndet} \to\Ifst$ is a task
  labeling for localization that assigns each singleton a unique label
  and all the other subsets are labeled the same.  
  
\end{example}

\BS{
   \begin{claim}
   Let $\imap_{ndet}$ and $\imapb$ be the I-maps defined in Example~\ref{ex:LshapedCorridor}. Then, $\imap_{ndet}$ is a minimal sufficient
  refinement of $\imapb$.
  \end{claim}

\begin{proof} Consider a subset $X' \subseteq X$ with cardinality $|X'|>1$ and some $x' \in X$ as labels assigned by $\imap_{ndet}$.
The I-map $\imapb$ is not sufficient because the transition corresponding to $([X']_{\imapb}, (u,y))$
can lead to multiple labels $[x']_\imapb$.
By construction $\imap_{ndet}$ is sufficient.  The I-map $\imap_{ndet}$ is a minimal sufficient refinement of $\imapb$ because it is sufficient and because there does not exist a sufficient $\imap$ such that $\imap_{ndet} \succ \imap \succeq \imapb$. Suppose to the contrary that a sufficient $\imap$ exists, which would mean that some equivalence classes could be merged. However, this is not possible because merging any of the non-singleton subsets violates sufficiency (as shown for $\imapb$) and merging singletons with others violates that it is a refinement.
\end{proof}}

A policy
  can be described over $\Ifshist/\imap_{ndet}$; $u=(1,0)$ starting
  from $X_0$ until $y_k=1$ is obtained and applying $u=(0,1)$ starting
  from $X_k$ until $y_n=1$ is obtained, then it is found that
  $q=(k,n)$ and $E$ is the corridor with $l_1=k$, $l_2=n$.

\subsection{Diversity-based Inference (DBI) as a derived ITS}
\label{ssec:DBI}

In this and the following section, we present DBI and its probabilistic counterpart PSR as deterministic ITSs. The core idea in DBI \citep{RivSch94, RivSch93} is to gather
information about the environment through action-observation
experiments. The environment is modeled as a finite state Moore
machine (finite state automaton), formally defined as a $6$-tuple
$\mathcal{E} = (X, U, f, h, Y, x_0)$. This definition coincides with
our definition of the external system, but also contains an
\emph{initial state} $x_0$.\endnote{In~\citep{RivSch94}, the authors
  allow several different observations (predicate symbols) $p \in P$ to
  be observed at each state via a function
  $\gamma \colon X \times P \to \{0, 1\}$. We assume for simplicity
  that each $x\in X$ corresponds to a unique observation $y = h(x)$.}
Experiments on $\mathcal{E}$ are called \emph{tests}. Each test
$\new{t} = (\tilde{u}_m, y)$ consists of a finite action sequence
$\tilde{u}_m := (u_1, \ldots, u_m) \in U^m$, followed by an
observation $y \in Y$. The test $\new{t} = (\tilde{u}_m, y)$ is said to
\emph{succeed} from state $x \in X$, if
$\big(h \circ f^{\tilde{u}_m} \big)(x) = y$, where
\begin{equation} \label{Eq_Def_iterated_action}
  f^{\tilde{u}_m}(x)
  := f\big( \cdots f\big(f(x, u_1), u_2\big) \ldots, u_m\big).
\end{equation}
By convention, if $m=0$, then $f^{\tilde{u}_m}(x)=x$.  Thus, for each
test $\new{t}$ there exists a \emph{success function}
$S_{\new{t}} \colon X \to \{0,1\}$, for which $S_\new{t}(x) = 1$ if and only
if $\new{t}$ succeeds at $x$. In DBI, an equivalence relation $\sim_\cT$
is defined in the set of tests
$\cT := \big\{ (\tilde{u}_m, y) \mid \tilde{u}_m \in U^m, \, y \in
Y, \, m \in \N \big\}$ by setting
\[
  \new{t}_1 \sim_\cT \new{t}_2 \quad \Longleftrightarrow
  \quad S_{\new{t}_1}(x) = S_{\new{t}_2}(x), \quad \forall x \in X.
\]
The cardinality $K := |P_\T|$ of the set of equivalence classes
$P_\T = \big\{[\new{t}] \mid \new{t} \in \cT \big\}$ is called the
\emph{diversity} of $\cE$. The diversity of a finite state machine
satisfies $K \leq 2^{|X|} < \infty$. Each state $x \in X$ can thus be
labeled by a finite \emph{success vector}
\begin{equation} \label{Eq_Def_Success_vector}
  \xi(x) := \big(S_1(x), \ldots, S_K(x) \big)
\end{equation}
\new{in which}, for $k \in \{1, \ldots, K\}$, the functions $S_k := S_{\new{t}_k}$ are the success functions of tests $\new{t}_1, \ldots, \new{t}_K$, \new{whose respective} equivalence classes $[\new{t}_1], \ldots, [\new{t}_K]$ \new{constitute the set}
$P_\cT$.

\begin{proposition}[\textbf{\new{The success vector is a minimal sufficient refinement}}]\label{prop:XiIsMSR}
  Let $\cE = (X, U, f, h, Y, x_0)$ be a finite state Moore machine \new{with diversity $K$}.
  Then  $\xi \colon X \to \new{\{0,1\}^K}$ defined in \eqref{Eq_Def_Success_vector}
 is a minimal sufficient refinement of~$h$.
\end{proposition}
\begin{proof}
  By setting $m=0$ in \eqref{Eq_Def_iterated_action}, we see that
  $\xi$ must be a refinement of~$h$. Suppose $\xi(x_0)=\xi(x_1)$ for
  some $x_0,x_1$, and let $u\in U$ be arbitrary. 
  We want to show that $\xi(f(x_0,u))=\xi(f(x_1,u))$.
  Let $(\tilde u_m,y)$ be any test and let $\tilde v_{m+1}=u\cat \tilde u_m$ denote the concatenation of $u$ as a prefix to $\tilde u_m$. 
  Then
  \begin{multline*}
       h\big(f^{\tilde u_m}(f(x_0,u))\big)= h\big(f^{\tilde v_{m+1}}(x_0)\big)=\\
       h\big(f^{\tilde v_{m+1}}(x_1)\big)=
       h\big(f^{\tilde u_{m}}(f(x_1,u)\big)
  \end{multline*}
  where the middle equality follows from the assumption that $\xi(x_0)=\xi(x_1)$.
  This means that all tests agree on $f(x_0,u)$ and $f(x_1,u)$, which implies that $\xi(f(x_0,u))=\xi(f(x_1,u))$. Hence, $\xi$ is sufficient.

  Suppose $\xi$ is not the minimal sufficient refinement of $h$.  Let $\xi'$ be a
  minimal sufficient refinement of $h$ which always exists, see
  Remark~\ref{rem:UniqueMSR}. Thus, we have
  $\xi\succ \xi'\succeq h$, so there are $x_0,x_1\in X$ with
  \begin{equation}
    \label{eq:xiprimeeq}
    \textrm{(i)} \,\,\, \xi'(x_0)=\xi'(x_1) \quad \text{and} \quad \textrm{(ii)} \,\,\,  \xi(x_0)\ne \xi(x_1).
  \end{equation}
  Since $\xi'$ is sufficient, it follows from~\eqref{eq:xiprimeeq}(i) that $\xi'(f(x_0, u)) = \xi'(f(x_1, u))$ for all actions $u \in U$. Using the sufficiency of $\xi'$ again $m$ times, it follows that
  \begin{equation} \label{Eq_xi_finite_observation_equality}
  \xi'\big(f^{\tilde u_m}(x_0)\big) = \xi'\big(f^{\tilde u_m}(x_1)\big)
  \end{equation}
  for all finite action sequences $\tilde{u}_m \in U^m$. Now, since $\xi'$ is a refinement of $h$, \eqref{Eq_xi_finite_observation_equality} implies that $h\big(f^{\tilde{u}_m}(x_0)\big) = h\big(f^{\tilde{u}_m}(x_1)\big)$ for all $\tilde{u}_m \in U^m$. By definition, this means that any test $(\tilde{u}_m,y)$ succeeds from $x_0$ if and only if it succeeds from~$x_1$. This implies $\xi(x_0) = \xi(x_1)$, contradicting \eqref{eq:xiprimeeq} (ii), and proves the claim. \qed
\end{proof}

\new{Now,} denote the concatenation of an action $u_0 \in U$ and a test
$\new{t} = (\tilde{u}_m, y) \in U^m \times Y$ by
\[
  u_0 \cat \new{t} := (u_0, u_1, \ldots, u_m, y) \in U^{m+1} \times Y.
\]
Since $S_\new{t} \big( f(x,u) \big) = S_{u \cat \new{t}}(x)$ for all
$\new{t} \in \T$ and $u \in U$, there exists for each $u \in U$ a
well-defined mapping $g_u \colon P_\cT \to P_\cT$ given by
$g_u\big([\new{t}]\big) := [u \cat \new{t}]$.  Furthermore, each $g_u$
defines a mapping (not necessarily a permutation)
$\alpha_u \colon\{1, \ldots, K\} \to \{1, \ldots, K\}$ by
\begin{equation} \label{Eq_Def_test_class_shuffle}
  \alpha_u(k) = n \quad \Longleftrightarrow \quad g_u\big([\new{t}_k]\big) = [\new{t}_n].
\end{equation}

\begin{definition}[\textbf{Update graph}]
  \label{Def_update_graph}
  Let $\mathcal{E} = (X, U, f, h, Y, x_0)$ be a finite state Moore
  machine with diversity $K = |P_\T|$. Let
  $P_\T = \big\{ [\new{t}_1], \ldots, [\new{t}_K] \big\}$ be the set of test
  equivalence classes with representatives
  $\new{t}_1, \ldots, \new{t}_K \in \T$, and let $S_k$ be the success
  function of test $\new{t}_k$.
Finally, let $\alpha_u$ be as in~\eqref{Eq_Def_test_class_shuffle}
  above.  The \emph{update graph} of $\cE$ is a state-relabeled \new{deterministic} transition system $\cG := (S, U, \new{\tau}, \sigma, Y, s_0)$ where $U$ and $Y$
  are as in $\cE$, and
  \begin{itemize}
  \item $S := \big\{ \big(S_1(x), \ldots, S_K(x) \big) \new{\in \{0,1\}^K} \mid x \in X  \big\}$,
  \item
    $\new{\tau}\big( (s_1, \ldots, s_K), u \big) := \big(s_{\alpha_u(1)},
    \ldots, s_{\alpha_u(K)} \big)$, 
  \item $\sigma\big( (s_1, \ldots, s_K) \big) = h(x)$, where $x \in X$
    \new{is such that} $s_k = S_k(x)$ for all $k \in \{1, \ldots, K\}$, and
    \item $s_0=(S_1(x_0),\dots,S_K(x_0))$.
  \end{itemize}
\end{definition}

A machine/environment $\cE$ is said to be \emph{reduced}, if for each
state $x \in X$ there exist tests $\new{t}_1, \new{t}_2 \in \cT$ for which
$S_{\new{t}_1}(x) \neq S_{\new{t}_2}(x)$. It was shown
in~\citep[\new{Theorem}~3]{RivSch94} that it is possible to simulate a reduced
environment $\cE$ by its update graph. We rephrase and prove this result in terms of isomorphisms of transition systems.
%
Two Moore machines
$(X, U, f, h, Y, x_0)$ and $(X', U, f', h', Y, x_0')$ (\new{both defined in terms of the same action and observation sets $U$ and $Y$}) 
are said to be \emph{isomorphic}, if there exists a bijective
map $g \colon X \to X'$ such that for all $x\in X$ and $u\in U$ we
have $f'(g(x),u)=g(f(x,u))$, $g(x_0)=x'_0$, and 
$(h'\circ g)(x)=h(x)$. The following is essentially
\citep[\new{Theorem} 3]{RivSch94}:

\begin{proposition}[\textbf{Update graph representation \new{of a Moore machine}}]
  \label{Prop_update_graph_isomorphism}
  Let $\cE = (X, U, f, h, Y, x_0)$ be a finite state Moore machine
  with a reduced state space $X$ and let $\cG := (S, U, \new{\tau}, \sigma, Y, s_0)$
  be the update graph of $\cE$. Then the function $\xi$ from~\eqref{Eq_Def_Success_vector}
  is an isomorphism between
  $\cE$ and $\cG$. 
\end{proposition}
\begin{proof}
  Let
  $P_\T = \big\{ [\new{t}_1], \ldots, [\new{t}_K] \big\}$ be the set of test
  equivalence classes in $\cE$ with representatives
  $\new{t}_1, \ldots, \new{t}_K \in \T$, and let $S_k$ be the success
  function of test $\new{t}_k$. Recall the definition of $\xi$
  from~\eqref{Eq_Def_Success_vector}.
  
  Then $\xi \colon X \to S$ is onto by definition of the set $S$. To
  show injectivity, assume $\xi(x) = \xi(x')$ for some $x, x' \in X$,
  which means that $S_k(x) = S_k(x')$ for all $k = 1,\ldots,K$. 
  Since $\cE$ is reduced, it suffices to show that
  $S_\new{t}(x) = S_\new{t}(x')$ for all tests $\new{t} \in \new{\T}$, because then
  $x = x'$.  Since every test $\new{t}$ satisfies $[\new{t}] = [\new{t}_k]$ for
  some $k$, the claim follows immediately.

  We still need to show that for all $(x,u) \in X \times U$, the
  functions $f_u(x) := f(x,u)$ and $\new{\tau}_u(s) := \new{\tau}(s, u)$ satisfy
  $\big( \new{\tau}_u \circ \xi\big)(x) = \big( \xi \circ f_u\big)(x)$.
According to Definition~\ref{Def_update_graph}, each $x \in X$ and $u \in U$ satisfies
  $(\new{\tau}_u \circ \xi)(x) = \big(S_{\alpha_u(1)}(x), \ldots,
  S_{\alpha_u(K)}(x) \big)$ where $\alpha_u(k) = n$ iff
  $[t_n] = [u \cat t_k]$. Thus,
  \begin{align} \label{Eq_commuting_T_xi_f}
    (\new{\tau}_u \circ \xi)(x) &= \big(S_{\alpha_u(1)}(x), \ldots, S_{\alpha_u(K)}(x) \big) \\
                       &= \big(S_{u \cat t_1}(x), \ldots, S_{u \cat t_K}(x) \big) \nonumber \\
                       &= \big(S_1(f(x,u)), \ldots, S_K(f(x,u)) \big) \nonumber \\
                       &= \xi \big( f(x,u) \big) = \big(\xi \circ f_u \big)(x). \nonumber
  \end{align}

  Since $\xi$ is a bijection, the labeling function $\sigma$ in
  Definition~\ref{Def_update_graph} is well defined, and satisfies
  $(\sigma \circ \xi)(x) = h(x)$ by definition. Finally, $\xi(x_0)=s_0$
  by the definition of~$s_0$. \qed
\end{proof}

To simulate the environment $\cE$ by $\cG$, one needs to set the
initial state $\hat{s}_0 = (s_1, \ldots, s_K) \in S$, which
corresponds to the initial state $x_0 \in X$ \new{for which $S_k(x_0) = s_k$ for all $k \in \{1, \ldots, K\}$.}

If we remove the assumption that $\cE$ is reduced, we can view the
function $\xi \colon X \to S$ in~\eqref{Eq_Def_Success_vector} as a
labeling that identifies those pairs of states that cannot be
differentiated by any test.

\begin{proposition}[\textbf{Update graph representation is a DITS}]
  \label{Prop_update_graph_minimality}
  Let $T_\cE = (X, U, f, h, Y)$ be the transition system corresponding
  to the finite state Moore machine $\cE = (X, U, f, h, Y, x_0)$, and
  let $\cG := (S, U, \tau, \sigma, Y)$ be the update graph of $\cE$.
  Then \new{$\cG$} is a DITS.
\end{proposition}
\begin{proof} \new{According to Proposition~\ref{prop:XiIsMSR}, $\xi: X \to S$ in~\eqref{Eq_Def_Success_vector} is a sufficient labeling for $T_\cE$. By definition, any two states (equivalence classes) $[x_1], [x_2]$ of the quotient system $T_\cE / \xi$ satisfy $[x_1] \neq [x_2]$ if and only if
$\xi(x_1) \neq \xi(x_2)$. This implies that $T_\cE / \xi$ defines a reduced Moore machine which is isomorphic to $\cG$ according to Proposition~\ref{Prop_update_graph_isomorphism}.
The claim then follows from Remark~\ref{rem:SuffDITS}, which states that}
  quotients by sufficient labelings
  are deterministic transition systems. \qed
\end{proof}

\subsection{Predictive state representations (PSRs)}

\emph{Predictive state representation (PSR)} \citep{LitSut01, JamSin04}, like its deterministic
predecessor DBI, is based on the idea of performing tests on the environment. The difference is that PSR assumes a statistical description of the internal-external coupling, expressed via success probabilities of tests, conditioned on past histories. PSRs have been shown to be more general than POMDPs \citep{CasLitZha97} in the sense that every POMDP model can be represented via the corresponding PSR.

Since the introduction of the original PSR, several variations of the concept have been proposed in connection to different learning algorithms that aim to discover the set of core tests and learn the associated prediction functions \citep{JamSin04, BooSidGor11, BooGreGor13}. So-called TPSRs \citep{RosGorThr04} are adaptations of the concept where, instead of maintaining vectors of probabilities over a finite set of core tests, a linear combination of a larger set of tests is maintained instead.


We focus on the original formulation of the PSR model and show how it can be expressed in our formalism as a DITS. Let $\eta_k := (\tilde{u}_{k-1}, \tilde{y}_k)$ denote the \emph{history
  I-state} at stage $k$ (including the $k$th observation, but not the
$k$th action). In addition let
$\tilde{y}_{k,m} := (y_k,\ldots,y_{k+m})$ and similarly
$\tilde{u}_{k,m} := (u_k,\ldots,u_{k+m})$ for all $m \in \N$. In PSR,
action-observation sequences
$\new{t} = \big(\tilde{u}_{m-1}^\new{t}, \tilde{y}_{m}^\new{t} \big)$ are
called \emph{tests}. At stage $k \in \N$, a PSR model maintains a
sufficient statistic for computing the conditional \emph{success probabilities}
\begin{equation} \label{Eq_PSR_cond_prob}
  P_{\eta_k}(\new{t}) := P \big(\tilde{y}_{k,m} = \tilde{y}_{m}^\new{t} \mid \tilde{u}_{k,m-1} = \tilde{u}_{m-1}^\new{t}, \, \eta_k \big)
\end{equation}
for tests $\new{t} = \big(\tilde{u}_{m-1}^\new{t}, \tilde{y}_{m}^\new{t} \big)$ of arbitrary length $m \geq 1$.

The idea of PSR is to identify minimal \emph{core sets of tests}
$Q = (\new{t}_1, \ldots, \new{t}_m)$ which have the property that, given any
test $\new{t} \notin Q$ and some history $\eta_k \in \Ifshist$, it is
possible to compute the success probability of $\new{t}$ as
$P_{\eta_k}(\new{t}) = f_\new{t}\big( Q(\eta_k) \big)$ where
$Q(\eta_k) := \big(P_{\eta_k}(\new{t}_1), \ldots, P_{\eta_k}(\new{t}_m)
\big)$ is the \emph{prediction vector} for the set $Q$ and $f_\new{t}$ is
the \emph{prediction function} associated with $\new{t}$. In
\emph{linear} PSR, the space of admissible prediction functions is
restricted to linear transformations (vectors) $r_\new{t} \in \R^{|Q|}$
so that $ f_\new{t}(Q(\eta_k)) = r_\new{t} \cdot Q(\eta_k)$ for all
$\new{t}, \eta_k$.

Formally, a PSR is a $5$-tuple $(U, Y, Q, F, m_0)$, where $U$ is the
set of actions, $Y$ is the set of observations, $Q$ is a \emph{core
  set of tests}, $F$ is the set of \emph{prediction functions}, and
$m_0 \in [0,1]^{|Q|}$ is the \emph{initial prediction vector} after
seeing the null history $\eta_0 = \varnothing$. A PSR model provides a
complete (probabilistic) description of the action-observation
dynamics, because the prediction vector $Q(\eta_k)$ can be updated
with each new action-observation. For this, only the (finite number
of) prediction functions $f_{(u,y)}$ and $f_{(u, y) \cat \new{t}_m}$ need
to be known, corresponding to all possible action-observation pairs
$(u,y)$ and to concatenations of these with the core tests
$\new{t}_m \in Q$. Then the update to $Q(\eta_{k+1})$, where
$\eta_{k+1} = \eta_k \cat (u,y)$, is obtained through the function
$\phi_{\textrm{PSR}} \colon [0,1]^m \times (U \times Y) \to [0,1]^m$
defined by
\begin{multline}
  \phi_{\textrm{PSR}}\big(Q(\eta_k), (u, y)  \big) := \\ \Big(\phi_1\big(Q(\eta_k), (u, y) \big), \ldots,
  \phi_m\big(Q(\eta_k), (u, y) \big) \Big),
\end{multline}
where the functions
$\phi_i \colon [0,1]^m \times (U \times Y) \to [0,1]$ are given for each
$i \in \{ 1, \ldots, m \}$ by
\begin{align*} \label{Eq_Def_PSR_update}
  \phi_i\big(Q(\eta_k), (u, y) \big) &:= P_{\eta_k \cat (u,y)}(\new{t}_i) \\
&= \frac{P_{\eta_k}\big((u,y) \cat \new{t}_i \big)}{P_{\eta_k}\big((u,y)\big)} = \frac{f_{(u, y) \cat \new{t}_i}\big(Q(\eta_k)\big)}{f_{(u, y)}\big(Q(\eta_k)\big)}.
\end{align*}
Thus, a PSR with a core set of tests $Q = (\new{t}_1, \ldots, \new{t}_m)$ is
a DITS $\big(\Ifspsr, U \times Y, \phi_{\textrm{PSR}} \big)$, where
$\Ifspsr := \big\{ Q(\eta_k) \mid \eta_k \in \Ifshist \big\}$. The
corresponding I-map $\kappa_\textrm{PSR} \colon \Ifshist \to \Ifspsr$ is
given by $\kappa_\textrm{PSR}(\eta) := Q(\eta)$.

\section{Conclusions and Future Work}\label{sec:discussion}

This paper introduced a mathematical framework for determining minimal filters and minimal feasible policies by comparing ITSs over information spaces. The minimality results are quite general without imposing strong restrictions on the underlying dynamical system (external system). We show that a large class of problems can be posed and analyzed under this framework. 

Nevertheless, there are several opportunities to expand the general theory.  For example, we assumed that $u$ is both the output of a policy and the actuation stimulus in the physical world; more generally, we should introduce a mapping from an action symbol $\sigma \in \Sigma$ to a control function $\histu \in \histU$ so that plans are expressed as $\pi\colon\ifs \rightarrow \Sigma$ and each $\sigma=\pi(\iota)$ produces energy in the physical world via a mapping from $\Sigma$ to~$\histU$.

It is also important to extend the models to continuous time.  In this case, the 
sensing and action histories are time parameterized functions, rather than sequences. 
Sufficiency must be defined in terms information mappings that apply to any time 
slice from $0$ to $t' < t$ for a history that runs from time $0$ to $t$, rather than 
only over discrete time steps. Some ground work has already been done in~\citep{Lav06}.

Another direction is to consider the hardware and actuation models as variables, and fix other model components. This is similar to the class of problems related to co-design for which the design process of a robot given resource constraints (sensors and actuators) is sought to be automated~\citep{SheOkaSab21,Cen16,ZarCenFra21}.

In this paper, \BS{we considered the theoretical limits on the DITS necessary to express a policy defined over a history ITS. However, the problem of finding such a DITS remains as an open algorithmic challenge.}
Furthermore, we only considered feasible policies.
An interesting direction is to analyze the information requirements for policies that are optimal with respect to a relevant objective and the trade-off between optimality and minimality. \BS{This will amount to an ordering (potentially a partial ordering) of policies in terms of (expected) cost and the minimal DITS to express such policy.}

In an external-internal coupled system, the different 
components, I-map $\kappa$, the 
information transition function $\phi$ and the policy $\pi$ share the total 
complexity (information content) of the internal information processing system. Of 
particular interest would be to explore the trade-offs between these components in 
terms of efficient encoding of data-structures and their successful decoding in terms 
of policies. Ultimately this could lead to fundamental characterisations of 
interaction system information content in the spirit of the minimum description 
length principle proposed in~\citep{Ris78}.

The mathematical theory of coupling as presented in this paper is very general.
Coupling in discrete dynamical systems, and of finite automata, are special cases of it, and even continuous systems can be seen as such.
Connections to other work on coupling such as~\citep{Spi15} is to be explored.
Dynamic coupling has been proposed as a viable approach to a mathematical
modelling of cognition from the enactivist perspective~\citep{MonHerZie08,Favela2020}.
The existing literatutre on the latter uses bits and pieces of dynamical
systems with sporadic applications in different areas of cognitive science,
but a systematic unifying study is still to be seen, especially one that
has meaningful ramifications to robotics and algorithmic design. An attempt
to connect philosophical ideas with those of this paper was presented
by the authors in~\citep{WeiSakLav22}.

A grand challenge remains: The results here are only a first step toward producing a more complete and unique theory of robotics that clearly characterizes the relationships between common tasks, robot systems, environments, and algorithms that perform filtering, planning, or learning.  We should search for lattice structures that play a role similar to that of language class hierarchies in the theory of computation.  This includes the structures of the current paper and the sensor lattices of \citep{Lav12b,ZhaShe21}.  Many existing filtering, planning, and learning methods can be formally characterized within this framework, which would provide insights into relative complexity, completeness, minimality, and time/space/energy tradeoffs.

\begin{dci}
The authors declared no potential conflicts of interest with respect to the research, authorship, and/or publication of this article.
\end{dci}

\begin{funding}
  This work was supported by a European Research Council Advanced
  Grant (ERC AdG, ILLUSIVE: Foundations of Perception Engineering,
  101020977), Academy of Finland (projects PERCEPT 322637, CHiMP
  342556), and Business Finland (project HUMOR 3656/31/2019).
\end{funding}

\theendnotes

\bibliographystyle{sageH}
\bibliography{Bibliography}
\end{document}